\documentclass[sigconf]{acmart}
\AtBeginDocument{%
 \providecommand\BibTeX{{%
  Bib\TeX}}}

\copyrightyear{2023}
\acmYear{2023}
\setcopyright{rightsretained}
\acmConference[KDD '23]{Proceedings of the 29th ACM SIGKDD Conference on Knowledge Discovery and Data Mining}{August 6--10, 2023}{Long Beach, CA, USA}
\acmBooktitle{Proceedings of the 29th ACM SIGKDD Conference on Knowledge Discovery and Data Mining (KDD '23), August 6--10, 2023, Long Beach, CA, USA}
\acmDOI{10.1145/3580305.3599272}
\acmISBN{979-8-4007-0103-0/23/08}

\AtBeginDocument{%
  \providecommand\BibTeX{{%
    \normalfont B\kern-0.5em{\scshape i\kern-0.25em b}\kern-0.8em\TeX}}}
\usepackage{xspace, color}
\usepackage{subfigure, multirow, tabularx} % For subfigure
\usepackage{booktabs} % For formal tables
\usepackage{enumitem}
\usepackage[T1]{fontenc}
\usepackage{epstopdf}
\usepackage{graphicx}
\usepackage{url}
\usepackage{wrapfig,lipsum}
\usepackage{makecell}
\usepackage{wrapfig}
\usepackage{array}
\usepackage{pifont}
\usepackage{float}
\usepackage{dsfont}
\usepackage{mathtools}
\usepackage{algorithm}
\usepackage{esint}
\usepackage[noend]{algpseudocode}

\newcommand{\nop}[1]{}

\newcommand{\model}{\mbox{\sc CF-GODE}\xspace}
\newcommand{\ode}{Treatment-Induced GraphODE\xspace}
\newcommand{\odelower}{treatment-induced GraphODE\xspace}

\newcommand{\ind}{\perp\!\!\!\!\perp} 
\newcommand{\CI}{\mathrel{\perp\mspace{-10mu}\perp}}
\newtheorem{theorem}{Theorem}
\newtheorem{theorem_copy}{Theorem}
\newtheorem{proposition}{Proposition}
\newtheorem{lemma}{Lemma}

\definecolor{red_color}{RGB}{255,0,0}
\newcommand{\redtext}[1]{\textcolor{red_color}{#1}}
\definecolor{yellow_color}{RGB}{255,199,44}

\definecolor{blue_color}{RGB}{39,116,174}
\newcommand{\bluetext}[1]{\textcolor{blue_color}{#1}}
\definecolor{dark_red}{RGB}{153, 31, 41}
\newcommand{\drtext}[1]{\textcolor{dark_red}{#1}}
\definecolor{green_color}{RGB}{91,175,52}

\definecolor{brown_color}{RGB}{205,90,161}

\definecolor{lg_color}{RGB}{63,147,139}

\definecolor{purple_color}{RGB}{165,86,57}

\definecolor{com_color}{RGB}{0,0,139}
\newcommand{\com}[1]{\textcolor{com_color}{#1}}

\begin{document}

%%
%% The "title" command has an optional parameter,
%% allowing the author to define a "short title" to be used in page headers.
% \title{\model: Estimating Continuous-Time Counterfactual \\
% Outcomes in Multi-Agent Dynamical Systems \YS{Potential Outcomes? As we discussed, there are two types of prediction: counterfactual and interventional.}\song{Let's discuss in today's meeting}\YS{Causal Inference with Continuous-Time Treatment for Multi-Agent Dynamical Systems}}
\title{CF-GODE: Continuous-Time Causal Inference for\\ Multi-Agent Dynamical Systems}

%%
%% The "author" command and its associated commands are used to define
%% the authors and their affiliations.
%% Of note is the shared affiliation of the first two authors, and the
%% "authornote" and "authornotemark" commands
%% used to denote shared contribution to the research.
\author{Song Jiang}
\email{songjiang@cs.ucla.edu}
\affiliation{%
  \institution{University of California, Los Angeles}
  \city{Los Angeles}
  \country{CA}
}

\author{Zijie Huang}
\email{zijiehuang@cs.ucla.edu}
\affiliation{%
  \institution{University of California, Los Angeles}
  \city{Los Angeles}
  \country{CA}
}

\author{Xiao Luo}
\email{xiaoluo@cs.ucla.edu}
\affiliation{%
  \institution{University of California, Los Angeles}
  \city{Los Angeles}
  \country{CA}
}

\author{Yizhou Sun}
\email{yzsun@cs.ucla.edu}
\affiliation{%
  \institution{University of California, Los Angeles}
  \city{Los Angeles}
  \country{CA}
}

%%
%% By default, the full list of authors will be used in the page
%% headers. Often, this list is too long, and will overlap
%% other information printed in the page headers. This command allows
%% the author to define a more concise list
%% of authors' names for this purpose.
\renewcommand{\shortauthors}{Song Jiang, Zijie Huang, Xiao Luo and Yizhou Sun.}

%%
%% The abstract is a short summary of the work to be presented in the
%% article.

\begin{abstract}
Multi-agent dynamical systems refer to scenarios where multiple units (aka agents) interact with each other and evolve collectively over time. For instance, people's health conditions are mutually influenced. Receiving vaccinations not only strengthens the long-term health status of one unit but also provides protection for those in their immediate surroundings. To make informed decisions in multi-agent dynamical systems, such as determining the optimal vaccine distribution plan, it is essential for decision-makers to \emph{estimate the continuous-time counterfactual outcomes}. However, existing studies of causal inference over time rely on the assumption that units are mutually independent, which is not valid for multi-agent dynamical systems. In this paper, we aim to bridge this gap and study how to estimate counterfactual outcomes in multi-agent dynamical systems. Causal inference in a multi-agent dynamical system has unique challenges: 1) Confounders are time-varying and are present in both individual unit covariates and those of other units; 2) Units are affected by not only their own but also others' treatments; 3) The treatments are naturally dynamic, such as receiving vaccines and boosters in a seasonal manner. To this end, we model a multi-agent dynamical system as a graph and propose a novel model called \model (\textbf{\underline{C}}ounter\textbf{\underline{F}}actual \textbf{\underline{G}}raph \textbf{\underline{O}}rdinary \textbf{\underline{D}}ifferential \textbf{\underline{E}}quations). \model is a causal model that estimates continuous-time counterfactual outcomes in the presence of inter-dependencies between units. To facilitate continuous-time estimation, we propose \ode, a novel ordinary differential equation based on graph neural networks (GNNs), which can incorporate dynamical treatments as additional inputs to predict potential outcomes over time. To remove confounding bias, we propose two domain adversarial learning based objectives that learn balanced continuous representation trajectories, which are not predictive of treatments and interference.
We further provide theoretical justification to prove their effectiveness. Experiments on two semi-synthetic datasets confirm that \model outperforms baselines on counterfactual estimation. We also provide extensive analyses to understand how our model works.

\end{abstract}


%%
%% Keywords. The author(s) should pick words that accurately describe
%% the work being presented. Separate the keywords with commas.

%%
%% This command processes the author and affiliation and title
%% information and builds the first part of the formatted document.
\maketitle

\section{Introduction}
Estimating counterfactual outcomes \emph{over time} is critical to gaining causal understanding for many useful practical applications, such as how to distribute the limited vaccines in the early days to maximize protection over time~\cite{medlock2009optimizing}, or how to design proper scheduling of medical treatments to optimize the patient recovery process~\cite{BicaAJS20}. Randomized controlled trials (RCTs) are the gold standard for causal inference, but they can be cost-prohibitive and ethically challenging, particularly when considering the dynamical settings described above. Therefore, estimating counterfactual outcomes from observational data is the key approach to answering causal questions in real-world scenarios. Existing research on observational causal inference over time has begun by utilizing basic linear regression~\cite{robins2000marginal} and Gaussian processes~\cite{xu2016bayesian} to capture the time-dependencies. Subsequently, advancements have been made by incorporating more advanced deep learning models such as recurrent neural networks (RNNs) ~\cite{BicaAJS20,fujii2022estimating} and Transformers ~\cite{MelnychukFF22}.

Despite the progress, all aforementioned studies have relied on the assumption that units (e.g., people in the vaccine example) are independent of each other, i.e., each unit is solely influenced by its own treatment but not by others. In many realistic scenarios, however, this assumption is not valid. For instance, a person's vaccination not only protects themselves but also those close to them. This type of setting is referred to as a \emph{multi-agent dynamical system}~\cite{gazi2007coordination}, where units (also known as agents) interact with each other and evolve collectively over time. Many practical problems can be expressed as multi-agent dynamical systems, such as the long-term effects of vaccination where people mutually influence~\cite{halloran2012causal}, brain network signals in which the regions of interest (ROI) in a brain are associated~\cite{yu2022learning,cui2022braingb}, and molecular systems movements where the atoms are interconnected ~\cite{durrant2011molecular}. Prior approaches for causal inference over time are not applicable to multi-agent dynamical systems since they are not capable of handling the interconnections between units. In this paper, we propose to study this novel problem \emph{counterfactual estimation in multi-agent dynamical systems}, which has received limited attention in the literature.  

\begin{figure}[t]
 \centering
 \includegraphics[width=1\columnwidth]{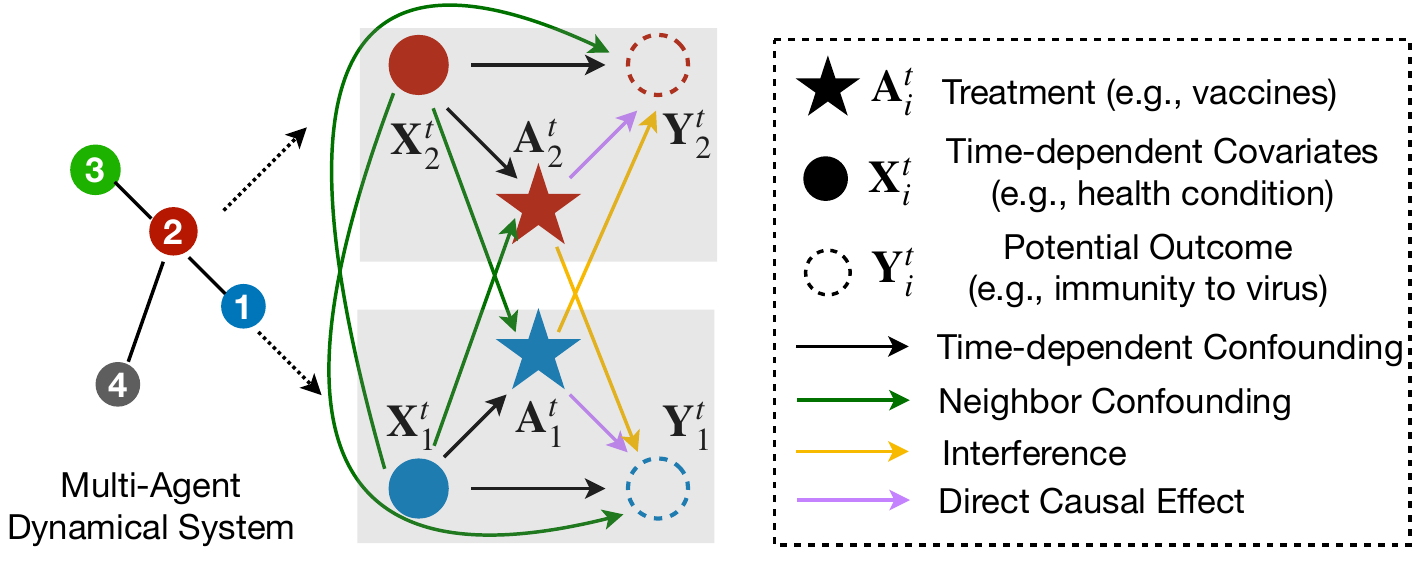}
  \caption{Causal graph at time $t$ in a multi-agent dynamical system. The causal variables are represented by shapes, while their relationships are distinguished by colors.  }\label{fig:intro}
   \vspace{-0.2cm}
\end{figure}

The dynamic and interrelated nature of multi-agent dynamical systems poses unique and nontrivial challenges to causal inference. We illustrate them along with Fig.~\ref{fig:intro}. 1) \textbf{Multi-source confounders}. Confounders are variables that have an impact on both treatments and outcomes, leading to spurious correlations between them. Therefore, in observational data, the treatments are not balanced among units with different confounders values, resulting in biased counterfactual outcomes estimation. For example, old people are more likely to receive vaccines, but also face a higher risk of virus infection. If we train a standard supervised model using such imbalanced data, it may wrongly predict that vaccines may increase the infection risk for young people. In multi-agent dynamical systems, the confounders are multi-source, including \emph{time-dependent confounders} and \emph{neighbor confounders}. Time-dependent confounders refer to the fact that the confounders typically evolve over time and thus their impact on treatments and outcomes also changes dynamically~\cite{platt2009time, BicaAJS20}. For instance, people's health conditions change over time, affecting their likelihood of getting vaccinated and future health status. Neighbor confounders mean that a unit's treatment and outcome could also be confounded by the covariates of its near units (neighbors)~\cite{forastiere2021identification,arbour2016inferring}, For example, if family members are in poor health, a unit may be more likely to receive a vaccine. Compared to the independent setting, neighbor confounders are additional confounding factors in multi-agent dynamical systems. 2) \textbf{Imbalance of interference}. As discussed in the previous example that vaccines protect not only a unit but also those in close proximity, the outcome of a unit can be influenced by others' treatments in multi-agent dynamical systems. In causal language, this phenomenon is referred to as \emph{interference}. Similar to the treatments, interference is affected by the covariates and thus is not balanced across the units in observational data~\cite{forastiere2021identification}. For instance, highly educated units are more likely to receive vaccines and typically have more highly educated friends. Therefore, they receive stronger protection through higher vaccination rates among their social networks. Such imbalanced interference causes additional bias in the estimation of counterfactual outcomes. 3) \textbf{Continuous dynamics}. In realistic applications, a multi-agent dynamical system is continuous in nature~\cite{porter2014dynamical}. However, most existing causal models are discrete, making them inappropriate for multi-agent dynamical systems. Modeling continuous-time observations (such as covariates and outcomes) and continuously estimating counterfactual outcomes over time remains an open challenge.

In this paper, we address the above challenges and study how to estimate continuous-time counterfactual outcomes, in presence of multi-source confounders and interference, in multi-agent dynamical systems. This is a novel, yet challenging and under-explored problem with valuable real-world applications. 

To this end, we model a multi-agent dynamical system as a \emph{graph}, where nodes represent units and edges capture their interactions. Inspired by recent achievements in graph ordinary differential equations (GraphODE)~\cite{huang2020learning}, we propose \model, a novel \emph{causal} model that estimates continuous-time \textbf{\underline{C}}ounter\textbf{\underline{F}}actual outcomes based on \textbf{\underline{G}}raph \textbf{\underline{O}}rdinary \textbf{\underline{D}}ifferential \textbf{\underline{E}}quations in multi-agent dynamical systems. Specifically, we use GraphODE as a backbone to model the continuous trajectory of each unit. However, in this case, traditional GraphODE can only model the pure dynamics of potential outcomes~\cite{huang2020learning} and lacks the ability to incorporate additional inputs such as treatments, making it inappropriate for causal inference. To address this issue, in \model, we propose \emph{\ode}, a new GraphODE model capable of handling treatments when predicting the future trajectory of potential outcomes. \ode uses graph neural networks (GNNs)~\cite{KipfW17} to formulate its differential equations, which can effectively capture the mutual dependencies between units including neighbor confounders and interference. This advantage makes it a natural fit for counterfactual estimation in multi-agent dynamical systems. Then a latent representation is learned for each unit from its observations as the solution to \ode, which represents the continuous trajectory driven by treatments. The core of ensuring \model is a causal model is to deal with the aforementioned estimation bias caused by imbalanced treatments and interference in the observational data. We solve this issue via domain adversarial learning~\cite{ganin2016domain,WangHK20}, in which we treat the values of treatments (and interference) as domains and ensure the latent representation trajectories are invariant to them. We provide theoretical justification to demonstrate that the domain-adversarial balancing objective functions proposed in \model can effectively achieve the balancing goal, thereby removing bias in counterfactual estimation and ensuring that \model is causal.

We summarize our major \textbf{contributions} as follows: 1) We study how to estimate counterfactual outcomes in multi-agent dynamical systems, which is a novel yet challenging problem with useful practical implications. 2) We propose \model, a novel causal model for causal inference multi-agent dynamical systems based on GraphODE and domain-adversarial learning. 3) We provide theoretical analysis to show that \model is able to handle the imbalanced treatments and interference, ensuring unbiased counterfactual estimation. 4)We conduct extensive experiments to evaluate \model's performance on counterfactual outcomes estimation in multi-agent dynamical systems.

\vspace{-10pt}
\section{Related Work}

\subsection{Causal Inference Over Time}
The central challenge in estimating counterfactual outcomes in longitudinal settings is to remove the confounding bias from time-dependent covariates~\cite{BicaAJS20,platt2009time}. The core solution to this in existing works is to cut off the association between covariates and the observed treatment assignments over time. To achieve this goal, statistical tools are widely used in traditional approaches. For example, marginal structural models (MSMs)~\cite{robins2000marginal} use inverse probability of treatment weighting (IPTW)~\cite{rosenbaum1983central,rosenbaum1987model} to balance the distribution of covariates over time between unit groups that are assigned to different treatments. By doing so, treatment assignment can no longer be predicted from the balanced covariates, thus breaking their correlation. A later work~\cite{lim2018forecasting} further enhances MSMs by using recurrent neural networks (RNNs) to learn the inverse probability of treatment weights (IPTWs), which is more capable of modeling sequential data. However, IPTW based tools can result in high variances in practice~\cite{BicaAJS20}. To overcome this limitation, recent studies~\cite{BicaAJS20, MelnychukFF22} extend the representation learning based balancing approaches from static settings~\cite{johansson2016learning,yoon2018ganite,shalit2017estimating,yao2018representation} to dynamic settings. Specifically, counterfactual recurrent network (CRN)~\cite{BicaAJS20} uses RNNs to encode the time-varying covariates into latent embeddings over time, which are simultaneously optimized by two objectives: potential outcomes prediction and longitudinal distribution balancing w.r.t. treatment assignments. The learned balanced embeddings are not predictive of treatments, thus ensuring unbiased estimates of the potential outcomes. Since RNNs are less powerful in capturing long-range dependencies, ~\cite{MelnychukFF22} improves CRN by using Transformer~\cite{vaswani2017attention} that preserves long-range dependencies between time-dependent confounders. Despite the progress, all the above models can only predict counterfactual outcomes in discrete timestamps. However, practical longitudinal sequences are continuous in nature.

Our work is most related to~\cite{gwak2020neural,bellot2021policy,SeedatIBQS22,de2022predicting}, which estimate counterfactual outcomes in continuous dynamic settings using neural ordinary differential equations (ODEs)~\cite{RubanovaCD19,chen2018neural} or neural controlled differential equations (CDEs)~\cite{kidger2020neural}. Specifically, ~\cite{SeedatIBQS22} infers continuous latent trajectories to represent the movement of potential outcomes and balance the distribution of this latent representation between treated and control groups via adversarial learning. However, ~\cite{SeedatIBQS22} (and all aforementioned models) assume that units are mutually independent, which is usually not valid in many practical scenarios where the units affect each other, e.g., getting vaccinated provides long-term protection not only for oneself but also for their close ones. In contrast, our model is designed to estimate counterfactual outcomes in longitudinal settings where units are interdependent, i.e., the multi-agent dynamical systems.

\subsection{Continuous Modeling With Neural Ordinary Differential Equations (ODEs)}
Many dynamical systems are continuous in nature, which can be typically modeled using first-order ordinary differential equations (ODEs)~\cite{porter2014dynamical}. ODEs describe a system's rate of change over time by a specific function, which is traditionally designed by domain experts~\cite{qian2021integrating}, and more recently parameterized by neural networks~\cite{RubanovaCD19,chen2018neural,massaroli2020dissecting}, as a closed-form ODE function may be unknown for some complex real-world systems. Given initial states, the solution of a NeuralODE can be easily computed using any ODE solver, such as the Runge-Kutta method~\cite{schober2019probabilistic}. In multi-agent dynamical systems, such as the spread of infectious disease among people, units often interact with one another, yet standard NeuralODEs do not explicitly model these interactions. Recent works have sought to address this limitation by representing the interactions among multiple units as graphs, and then utilizing graph neural networks (GNNs)~\cite{KipfW17,VelickovicCCRLB18} to parameterize the ODE function~\cite{huang2021coupled,huang2020learning,ZangW20a,poli2019graph}. When predicting the dynamics of each unit, these GraphODE models not only take into account the unit's own latent state but also aggregate the latent states of its connected units along the interaction graph, to effectively capture the mutual influence between them. However, GraphODE models are standard statistical methods and therefore lack the capability for causal inference. Instead, our model aims to address the unique challenges present in multi-agent dynamical systems, i.e., time-dependent confounders and network interference, in order to make counterfactual predictions.

\section{Problem Setup}

\subsection{Problem Formulation}\label{sec:problem}

We study how to estimate counterfactual outcomes in the context of multi-agent dynamical systems, where the units engage in mutual interactions and evolve simultaneously over time. Throughout this paper, we use boldface uppercase letters to denote matrices or vectors, boldface uppercase letters with subscripts to signify elements of matrices or vectors, regular lowercase letters to represent values of variables, and calligraphic uppercase letters to indicate sets. We summarize all notations used in this paper in Appendix.~\ref{sec:notations}.

Formally, a multi-agent dynamical system can be represented by a dynamical graph $\mathcal{G}^t=(\mathcal{V},\mathcal{E}^t)$, where $\mathcal{V} = \{v_1, v_2, ..., v_N\}$ is the set of $N$ units (nodes) and $\mathcal{E}^t$ denotes the edge set at time $t$. An edge in $\mathcal{E}^t$ describes the intersection between the two units it connects at time $t$. In this paper, we present an early exploration of causal inference in multi-agent dynamical systems, and for the purpose of simplicity, we assume that the graph structure remains constant over time, i.e., $\mathcal{G}^t = \mathcal{G}$. Each unit is associated with time-varying variables, which are the causal quantities in our case. We introduce them together with the causal framework in the following.

We follow the longitudinal potential outcomes framework~\cite{robins2009estimation,rubin1978bayesian} to formalize the counterfactual outcome estimation as in~\cite{BicaAJS20, SeedatIBQS22}. The observational data $\left(\left( \mathbf{X}^t, \mathbf{A}^t, \mathbf{Y}^t \right) \cup \mathbf{V}\right)$ in a multi-agent dynamical system contains time-dependent covariates $\mathbf{X}^t$ (e.g., health condition), dynamical
treatments $\mathbf{A}^t$ (e.g., vaccine allocation), and time-varying outcomes $\mathbf{Y}^t$ (e.g., immunity to infectious disease). It is worth noting that $\mathbf{Y}^t$ is essentially a part of $\mathbf{X}^t$. $\mathbf{V}$ denotes the static covariates of units such as ethnicity. Let the historical records of the multi-agent dynamical system up to time $t$ be represented by $\mathcal{H}^t = \{\mathbf{\Bar{X}}^t,\mathbf{\Bar{A}}^t,\mathbf{\Bar{Y}}^t, \mathbf{V}\}$, where $\mathbf{\Bar{X}}^t, \mathbf{\Bar{A}}^t,\mathbf{\Bar{Y}}^t$ are all the $\mathbf{X}^{t^-}, \mathbf{A}^{t^-}, \mathbf{Y}^{t^-}$ until $t$ $(t^-\leq t)$, respectively. In causal inference, we are focused on understanding the potential outcomes $\mathbf{Y}^{t^+}(\mathbf{A}^{t^+}=a)$\footnote{The potential outcome can also be formalized using \emph{do} operation~\cite{pearl2009causality}.} that may occur in the future $(t^+>t)$ under a specific treatment $a$, which explains the impact of the treatment assignment on the dynamics of the system. Note that $a$ is a treatment trajectory that includes all treatments in the future time. Our goal is to estimate the future potential outcomes sequence driven by treatments in a multi-agent dynamical system, which is formalized as:
\begin{align}
\mathbb{E}\left(\mathbf{Y}^{t^+}(\mathbf{A}^{t^+}=a)\mid \mathcal{H}^t, \mathcal{G}\right).\label{eqn:formalization}
\end{align}
\subsection{Causal Identification}
The potential outcomes represented by $\mathbf{Y}^{t^+}(\mathbf{A}^{t^+}=a)$ are a causal quantity. To make it identifiable from observational data, we must adhere to the following necessary assumptions.

\textbf{Assumption 1: Positivity (Overlap)}. The future treatment trajectory is probabilistic regardless of the historical observation, i.e., $0 < P(\mathbf{A}^{t^+}=a\mid \mathcal{H}^t) < 1, \forall \mathcal{H}^t$.

\textbf{Assumption 2: Consistency}. Under the same treatment trajectory $a$, the potential outcome is equal to the observed outcomes, i.e., $\mathbf{Y^{t^+}}(\mathbf{A}^{t^+}=a) = Y^{t^+}$.

The above two assumptions are standard for longitudinal counterfactual estimation. To identify the potential outcomes, it is also necessary to assume that there are no unobserved confounders, i.e., the strong ignorability assumption. However, the typical sequential strong ignorability assumption~\cite{BicaAJS20,SeedatIBQS22,MelnychukFF22,lim2018forecasting} is not appropriate for multi-agent dynamical systems, because the graph structure $\mathcal{G}$ introduces extra graph confounders and interference. A plausible strong ignorability assumption for graphs is first introduced by~\cite{forastiere2021identification} and later validated in studies such as~\cite{ma2021causal, ma2022learning, jiang2022estimating}. However, the assumption made in these works is limited to static settings. To address this, we extend it to longitudinal settings and adapt it to be applicable to multi-agent dynamical systems in the following.

\begin{figure}[t]
 \centering
 \includegraphics[width=1\columnwidth]{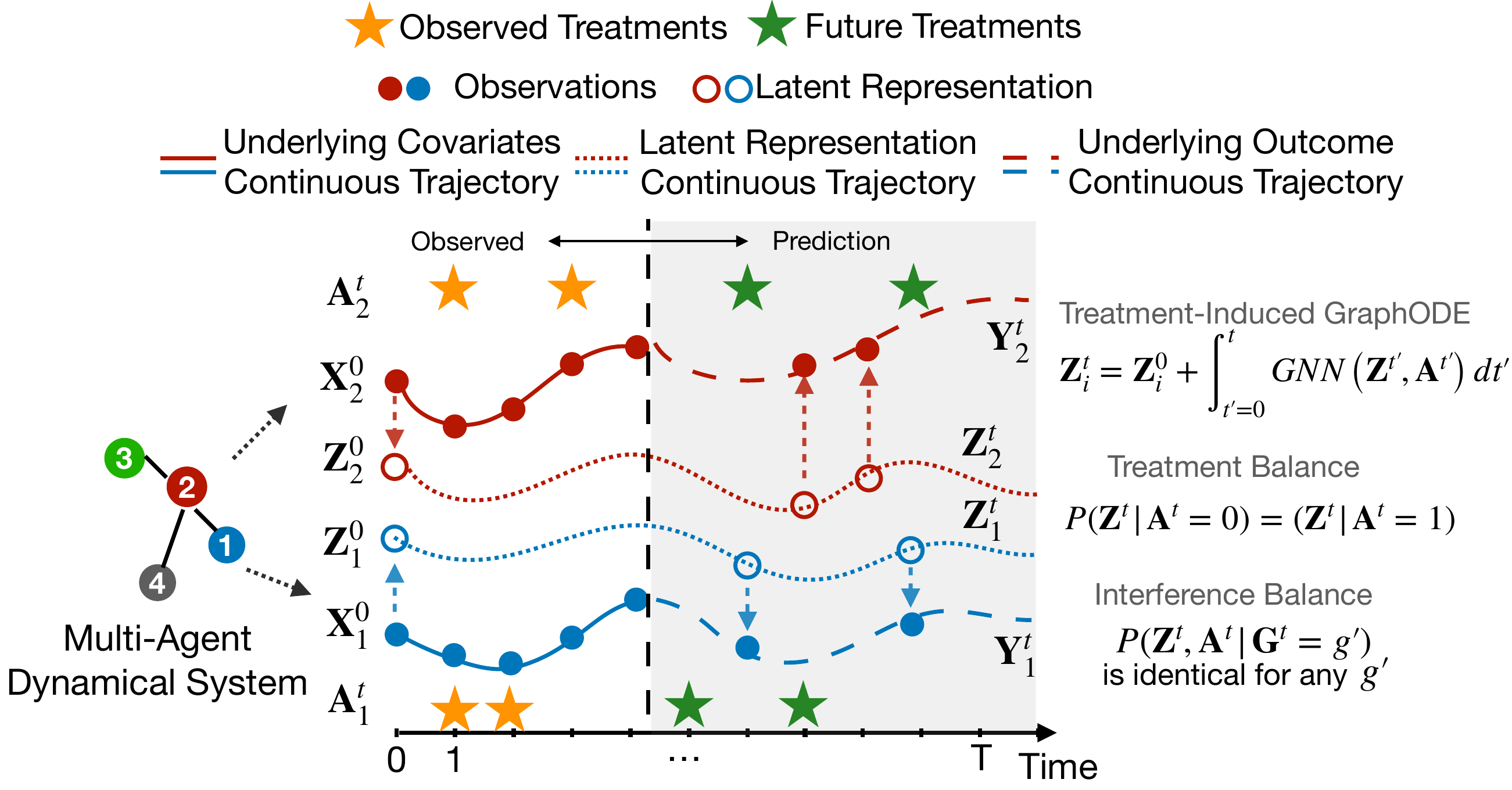}
  \caption{Overview of \model. The initial latent representation $\mathbf{\mathbf{Z}^0}$ is first learned from initial observations. Then the continuous latent representation trajectory $\mathbf{\mathbf{Z}^t}$ is learned as the solution to \odelower, which is able to handle treatments as additional inputs. The graph neural network (GNN) based ODE function naturally models the mutual dependencies. The future potential outcomes can be decoded from $\mathbf{\mathbf{Z}^t}$ at any given time. To remove confounding bias, $\mathbf{\mathbf{Z}^t}$ is balanced with respect to 1) treatments, 2) interference when combined with corresponding treatments.}\label{fig:frame}
  \vspace{-15pt}
\end{figure}

We first introduce a summary function, denoted as $g(\cdot)$, that captures the interference effects caused by the treatments of a node's neighboring units in the graph as in~\cite{forastiere2021identification}. Formally, $\mathbf{G}_i^t = g(\mathbf{A}^t_{\mathcal{N}_i},\mathbf{A}^t_{\mathcal{N}_{-i}})$, where $\mathbf{A}^t_{\mathcal{N}_i}$ denotes the treatments of node $i$'s immediate neighbors, and $\mathbf{A}^t_{\mathcal{N}_{-i}}$ is the treatments of all the remaining node that are not directly connected to node $i$. We refer to $\mathbf{G}_i^t$ as \emph{interference summary}. Here for simplicity, we adopt the assumption put forth in ~\cite{forastiere2021identification,arbour2016inferring,jiang2022estimating} that a node is only influenced by the treatments of its immediate neighbors, i.e., $g(\mathbf{A}^t_{\mathcal{N}_i},\mathbf{A}^t_{\mathcal{N}_{-i}}) = g(\mathbf{A}^t_{\mathcal{N}_i},\mathbf{A'}^t_{\mathcal{N}_{-i}}) = g(\mathbf{A}^t_{\mathcal{N}_i}), \forall \mathbf{A}^t_{\mathcal{N}_{-i}}, \mathbf{A'}^t_{\mathcal{N}_{-i}}$. $g(\cdot)$ can be instantiated using any aggregation functions or models. As in previous studies~~\cite{forastiere2021identification,ma2021causal,jiang2022estimating}, in this paper, we define $\mathbf{G}_i^t$ as the proportion of treated units in unit $i$'s neighbors, i.e., $\mathbf{G}_i^t \vcentcolon=\sum_{j\in \mathcal{N}_i}\frac{\mathbf{A}_i^t}{|N_i|}$.
With $\mathbf{G}_i^t$, we present the strong ignorability assumption for multi-agent dynamical systems in the following:

\textbf{Assumption 3: Strong Ignorability for Multi-Agent Dynamical Systems}\footnote{Note that similar to the strong ignorability assumptions in static or non-graph sequential settings, assumption 3 can not be verified only from data.}. Given the historical observations and the graph structure that describes the multi-agent dynamical system, the potential outcome trajectory is independent of the treatments and interference summary, i.e., $\mathbf{Y^{t^+}}(\mathbf{A}^{t^+}=a) \CI \mathbf{A}^{t^+}, \mathbf{G}^{t^+} \mid \mathcal{H}^t, \mathcal{G}. \forall a, t$.

With these three assumptions, the potential outcome trajectory Eq.~(\ref{eqn:formalization}) can be identifiable as:
% \begin{proof}
\begin{align}
        &\mathbb{E}\left(\mathbf{Y}^{t^+}(\mathbf{A}^{t^+}=a)\mid \mathcal{H}^t, \mathcal{G}\right)\nonumber \\
        &= \mathbb{E}\left(\mathbf{Y}^{t^+}(\mathbf{A}^{t^+}=a)\mid \mathbf{A}^{t^+}, \mathbf{G}^{t^+}, \mathcal{H}^t, \mathcal{G}\right)\label{eqn:identify_ignore}\\
&= \mathbb{E}\left(\mathbf{Y}^{t^+}\mid \mathbf{A}^{t^+}, \mathbf{G}^{t^+}, \mathcal{H}^t, \mathcal{G}\right)\label{eqn:identify_consist}.
\end{align}
Eq.~(\ref{eqn:identify_ignore}) is true because of assumption 3, while Eq.~(\ref{eqn:identify_consist}) holds under the assumption 2.
% \end{proof}
The above causal identification enables us to estimate the potential outcomes in multi-agent dynamical systems using observational data. More specifically, we can train a machine learning model on observational data, which takes treatment trajectory $\mathbf{A}^{t^+}$, interference summary $\mathbf{G}^{t^+}$, historical observation $\mathcal{H}^t$ and graph $\mathcal{G}$ as inputs, and the observed (factual) outcome $\mathbf{Y}^{t^+}$ as targets, to predict the counterfactual outcomes given new treatment trajectories. Our proposed model \model is grounded in this and will be presented in detail in the subsequent section.
\section{Proposed Model: \model}

\subsection{Overview}
Our proposed \model is a causal model that predicts counterfactual outcomes in a multi-agent dynamical system by learning from observational data. We show an overview of our model in Fig.~\ref{fig:frame}. Compared to most existing causal models designed for standard sequential settings that consider discrete time intervals and independent units~\cite{BicaAJS20, MelnychukFF22}, multi-agent dynamical systems are more realistic and present two challenging properties: the dynamics are \emph{continuous} in nature, and units are \emph{influenced} by others. To address these, our proposed \model takes the advantage of recent breakthroughs in graph ordinary differential equations (GraphODE)~\cite{huang2020learning,huang2021coupled} and extends it to handle treatments and interference, enabling continuous estimation of counterfactual outcomes in multi-agent dynamical systems. We refer to our ODE model as \emph{\ode} (Sec.~\ref{sec:graph_ode}). The time-dependent confounders lead the distribution of covariates to be quite discrepant between units assigned to different treatments, resulting in high variances in counterfactual outcome estimation~\cite{johansson2016learning, shalit2017estimating,robins2000marginal}. This effect is further amplified by the imbalanced interference caused by the graph structure in multi-agent dynamical systems~\cite{forastiere2021identification,jiang2022estimating}. \model uses adversarial learning to alleviate this issue and guarantee unbiased estimates of counterfactual outcomes (Sec.~\ref{sec:adversarial_learning}).

\subsection{\ode}\label{sec:graph_ode}
To facilitate continuous-time counterfactual outcome estimation, we propose to learn a continuous latent trajectory $\mathbf{Z}_i^t$ for every node in multi-agent dynamical system that represents their movement. An ideal $\mathbf{Z}_i^t$ should possess two characteristics: 1) the ability to predict observed outcomes, and 2) not to be predictive of the received treatment or interference in observational data\footnote{The second characteristic is discussed in Sec.~\ref{sec:adversarial_learning}.}. We implement such a $\mathbf{Z}_i^t$ by a novel model called \ode, which empowers the recent GraphODE~\cite{huang2020learning,huang2021coupled} to deal with treatment and interference for counterfactual outcomes estimation. 

In a multi-agent dynamical system, the future outcomes of node $i$ might be affected by not only its own past movement and current treatment, but also the movements and interference from neighbors (e.g., a unit's health condition and vaccination status have a significant impact on how likely others are to be infected). We model this process and formalize \odelower as:
\begin{align}
    \mathbf{Z}_{i}^{t}=\mathbf{Z}_{i}^{0}+\int_{t'=0}^{t} \phi\left(\mathbf{Z}^{t'},\mathbf{A}^{t'}\right) dt'\label{eqn:ode}.
\end{align}
In Eq.~(\ref{eqn:ode}), $\mathbf{Z}^{t'}$ and $\mathbf{A}^{t'}$ denote the latent trajectory representations and treatments of all nodes in the multi-agent dynamical system, respectively. $\phi(\cdot)$ is the ODE function. To comprehensively capture the effects from node $i$ and its connected neighbors, we parameterize $\phi(\cdot)$ using graph neural networks~\cite{KipfW17} with self-loops. $\mathbf{Z}_{i}^{0}$ is the initial state and can be encoded from the initial observations as  $\mathbf{Z}_{i}^{0} = f(\mathbf{X}_{i}^{0},\mathbf{V}_i)$, where $f(\cdot)$ is an encoder parameterized by neural networks. With $\mathbf{Z}_{i}^{0}$, we can obtain the $\mathbf{Z}_{i}^{t}$, which is the solution to \odelower, by solving an ODE initial-value problem (IVP) in Eq.~(\ref{eqn:ode}), formalized as:
\begin{align}
     \mathbf{Z}_{i}^{0},\mathbf{Z}_{i}^{1}\cdots \mathbf{Z}_{i}^{T} = \text{ODESolve}\left(\phi,[\mathbf{Z}_1^0,\mathbf{Z}_2^0\cdots \mathbf{Z}_N^0],\left(t_0,t_1\cdots t_T\right)\right)\label{eqn:solver},
\end{align}
where $T$ is the number of timestamps for the evaluation of Eq.~(\ref{eqn:solver}). With the solution latent trajectory $\mathbf{Z}_{i}^{t}$, we can then use a decoder $d_\mathbf{Y}(\cdot)$ to transform it to the predicted outcome $\hat{\mathbf{Y}}_i^t = d_\mathbf{Y}(\mathbf{Z}_i^t)$. We also use neural networks to instantiate $d_\mathbf{Y}(\cdot)$. We compare the prediction $\hat{\mathbf{Y}}_i^t$ to ground-truths $\mathbf{Y}_i^t$ in all observed timestamps $\left(t_0,t_1\cdots t_T\right)$ using a mean square error as objective, which is formalized as:
\begin{align}
    L^{\langle Y \rangle} = \frac{1}{N}\frac{1}{T}\sum_i^N\sum_t^T\left(\hat{\mathbf{Y}}_i^t - \mathbf{Y}_i^t\right)^2.
\end{align}

\subsection{Balancing via Adversarial Learning}\label{sec:adversarial_learning}
In the observational data, the treatments applied to each unit $\mathbf{A}_i^t$ are affected by the time-dependent confounders present in the covariates (and thus in its latent representation trajectory $\mathbf{Z}_i^t$). Consequently, the distribution of latent representation trajectory is not balanced among units with different treatment assignments, i.e., $P(\mathbf{A}_i^t|\mathbf{Z}_i^t)$ is not uniform, leading to high variances in the counterfactual outcome estimation~\cite{johansson2016learning,shalit2017estimating}. In the context of multi-agent dynamical systems, this effect is further exacerbated by the presence of \emph{imbalanced interference} among units. This is because a unit's interference is influenced by its covariates (also the latent representation) and treatments in the observational data, i.e., $P(\mathbf{G}_i^t|\mathbf{Z}_i^t,\mathbf{A}_i^t)$ is not uniform~\cite{forastiere2021identification,jiang2022estimating,ma2021causal}. Here we give an intuitive example of the imbalanced interference: consider that a highly educated person is more likely to be surrounded by other highly educated friends, who believe in science and are more likely to be vaccinated, thereby providing stronger protection for this person against infectious diseases, i.e., higher interference.

A sufficient condition to remove the above bias is to ensure that the distribution of latent representation trajectories is invariant to treatments, and when combined with the corresponding treatments, is interference-invariant~\cite{shalit2017estimating,BicaAJS20,SeedatIBQS22,forastiere2021identification,jiang2022estimating}. This condition is formalized as  $P(\mathbf{Z}^t|\mathbf{A}^t=0) = (\mathbf{Z}^t|\mathbf{A}^t=1)$ for treatment balancing, and $P(\mathbf{Z}^t,\mathbf{A}^t|\mathbf{G}^t=g')$ is identical for any given value of $g'$ for interference balancing. The treatment $\mathbf{A}^t$ is binary and interference $\mathbf{G}^t$ is continuous as in~\cite{forastiere2021identification,jiang2022estimating}. Note that the aforementioned conditions are over the unit groups. This guarantees that the treatment cannot be inferred from the latent representation trajectory, and that the interference is not predictable when the treatment is combined with latent representation. We implement this balancing goal through domain adversarial learning~\cite{ganin2016domain}, in which the treatment is treated as binary domains and the interference is treated as continuous domains~\cite{WangHK20}. Specifically, we use the gradient reversal layer proposed in~\cite{ganin2016domain}, denoted as $r(\cdot)$, to adversarially optimize the latent representation trajectory at every observed time, making it agnostic towards the treatments and interference. 

\textbf{Treatment Balancing.} Formally, the predicted treatment is $\hat{\mathbf{A}}_i^t = d_\mathbf{A}\left(r(\mathbf{Z}_i^t)\right)$, where the $d_\mathbf{A}$ is a neural network that attempts to recover the treatment from latent representation. The gradient reversal layer $r(\cdot)$ does nothing in the forward pass, but reverses the gradients in the back-propagation. This way, a min-max game is created in which $d_\mathbf{A}$ aims to minimize the treatment prediction loss, while the latent representation learner in \odelower strives to maximize it, as formalized in the following:
\begin{align}\label{eqn:a_pred}
    L^{\langle A \rangle} = \underset{d_\mathbf{A}^j}{\text{min}}\ \underset{f,\phi}{\text{max}}\frac{1}{N}\frac{1}{T}\sum_i^N\sum_t^T\sum_{j\in\{0,1\}}\mathds{1}_{(\mathbf{A}_i^t=j)}-\log\left(d_\mathbf{A}^j\left(r(\mathbf{Z}_i^t)\right)\right),
\end{align}
where $d_\mathbf{A}^j$ represents the logits of $d_\mathbf{A}(\cdot)$ for predicting treatment $j$. We then provide a theoretical analysis to justify the capability of $L^{\langle A \rangle}$ to attain balanced representations in the following.
\begin{theorem}\label{theorem:t1}
 Let $j\in \{0,1\}$ be the binary treatment values, and let $N$ and $T$ denote the number of units and observed timestamp lengths, respectively. Let $P_j^t = P(\mathbf{Z}^t\mid \mathbf{A}^t=j)$, be the distribution of latent representation $\mathbf{Z}^t$ for the group of units with treatments $j$ at time $t$. Let $f$, $\phi$, $d_\mathbf{A}^j$ be the initial state encoder, the ODE function of \odelower, and logits of predicting treatment $j$. The necessary and sufficient condition for the min-max game in Eq.~(\ref{eqn:a_pred}) to be optimal is $P_0^t = P_1^t, \forall t\in\left(t_0,t_1\cdots t_T\right)$.
 \end{theorem}
Theorem~\ref{theorem:t1} suggests that the condition to obtain global optimum of Eq.~(\ref{eqn:a_pred}) is $P(\mathbf{Z}^t|\mathbf{A}^t=0) = (\mathbf{Z}^t|\mathbf{A}^t=1)$. Therefore, by optimizing $L^{\langle A \rangle}$ in Eq.~(\ref{eqn:a_pred}), we can ensure the latent representation trajectory $\mathbf{Z}^t$ is balanced with respect to treatments. In other words, $\mathbf{Z}^t$ is not predictive of  $\mathbf{A}^t$.  We prove Theorem~\ref{theorem:t1} in Appendix.~\ref{sec:proof_t1}.

\textbf{Interference Balancing. } The interference $\mathbf{G}_i^t$ is continuous, we thus adapt the continuous domain adversarial learning~\cite{WangHK20} to achieve the interference balancing. Similar to the binary case, we consider the continuous interference as continuous domains, and use the gradient reversal layer $r(\cdot)$ to build a min-max game on interference prediction as follows:
\begin{align}\label{eqn:pred_g}
    L^{\langle G \rangle} &= \underset{d_\mathbf{G}}{\text{min}}\ \underset{f,\phi}{\text{max}}\frac{1}{N}\frac{1}{T}\sum_i^N\sum_t^T\left(d_\mathbf{G}\left(r([\mathbf{Z}_i^t,\mathbf{A}_i^t])\right)-\mathbf{G}_i^t\right)^2    
\end{align}
where $d_\mathbf{G}$ is the interference predictor which is parameterized by neural networks, and $[\cdot,\cdot]$ is the concatenation operation. In the following, we also theoretically demonstrate that $L^{\langle G \rangle}$ is able to achieve the interference balancing objective.

\begin{theorem}\label{theorem:t2}
Let $f$, $\phi$, $d_\mathbf{G}$ be the initial state encoder, the ODE function of \odelower, and the interference predictor. The necessary and sufficient condition for min-max game in Eq.~(\ref{eqn:pred_g}) to be optimal is $P(\mathbf{Z}^t,\mathbf{A}^t|\mathbf{G}^t=g')$ is identical for any $g'$.
\end{theorem}

Theorem.~\ref{theorem:t2} indicates that if $\mathbb{E}([\mathbf{Z}^t,\mathbf{A}^t]\mid \mathbf{G}^t)$ is identical for any $\mathbf{G}^t=g'$, Eq.~(\ref{eqn:pred_g}) achieves optimum. Therefore, it is sufficient to balance the combination of representations and treatments with respect to interference $\mathbf{G}^t$ by optimizing the objective function $L^{\langle G \rangle}$. We show the proof of Theorem~\ref{theorem:t2} in Appendix.~\ref{sec:proof_t2}.

\subsection{Training of \model}
\textbf{Objective Function. }The overall objective function of \model is formalized in the following:
\begin{align}
    L = L^{\langle Y \rangle} + \alpha_{\mathbf{A}} L^{\langle A \rangle} + \alpha_{\mathbf{G}} L^{\langle G \rangle},
\end{align}
where coefficients $\alpha_{\mathbf{A}}$, $\alpha_{\mathbf{G}}$ are the strengths of the treatment balancing and interference balancing, respectively. By adversarially optimizing $L$, the latent representation trajectory $\mathbf{Z}_i^t$ is able to predict the outcome trajectory $\mathbf{Y}_i^t$ while remaining invariant to the treatments $\mathbf{A}_i^t$ and interference $\mathbf{G}_i^t$ (combined with treatments), which enables the unbiased counterfactual outcome estimation in multi-agent dynamical systems. 

\textbf{Alternative Training as Trade-Off. } In practice, we find that directly training \model with the overall loss function $L$ may not be stable as $L^{\langle A \rangle}$ and $L^{\langle G \rangle}$ could hinder the ability of latent representation trajectory $\mathbf{Z}_i^t$ to predict the outcome. Therefore, we trade-off the training of \model in an alternative manner between $L$ and $L^{\langle Y \rangle}$, to ensure that $\mathbf{Z}_i^t$ is capable of predicting outcomes. Specifically, we switch the training iterations between $L$ and $L^{\langle Y \rangle}$ with a ratio of $K$, i.e., $\frac{Iter_{L}}{Iter_{L^{\langle Y \rangle}}} = K$, where $Iter$ means the number of training iterations and $K$ is a tunable hyperparameter. We elaborate on the training procedure in Appendix.~\ref{sec:pseudo}.

\section{Experiments}

\subsection{Experimental Settings}
\textbf{Dataset. } In observational data, we only have factual outcomes but not counterfactual outcomes. Therefore, we use semi-synthetics data to evaluate \model as in~\cite{ma2022learning,guo2020learning,veitch2019using}. That is, we use two real graphs Flickr and BlogCatalog~\cite{guo2020learning,chu2021graph,ma2021deconfounding} and use a Pharmacokinetic-Pharmacodynamic (PK-PD) model~\cite{goutelle2008hill} to simulate the continuous trajectory of treatments and potential outcomes~\cite{BicaAJS20,SeedatIBQS22}. The data simulation mimics the vaccine example in the real world. We introduce  the data simulation process in detail in Appendix.~\ref{sec:exp_setting}.  

\textbf{Metric. } We focus on counterfactual outcomes estimation in this paper, which is a continuous value. Therefore, we use mean square errors (MSE) as our metric to evaluate the performance of our model, which is formalized as $MSE:= \frac{1}{N}\frac{1}{T}\sum_i^N\sum_t^T\left(\hat{\mathbf{Y}}_i^t - \mathbf{Y}_i^t\right)^2$.

\textbf{Baselines. } The scope of our model is in continuous-time causal inference, therefore we compare \model with the following baselines: \textbf{CDE}~\cite{kidger2020neural}: Ordinary differential equations with external inputs to adjust the continuous trajectory. \textbf{GraphODE}~\cite{huang2020learning} Ordinary differential equations model with graph neural networks (GNNs) based ODE functions. \textbf{TE-CDE}~\cite{SeedatIBQS22}:  the state-of-the-art model for continuous-time counterfactual outcomes estimation based on neural controlled differential equations (NeuralCDE).

\textbf{Implementation. } The parameters of \model are set as follows: the dimension of latent representations is $64$; the ODE solver is the Euler method; the balancing degrees are $\alpha_{\mathbf{A}}=\alpha_{\mathbf{G}}=0.5$. For training hyperparameters, the learning rate is $0.0001$; the default alternative training ratio $K$ is $4$. We train the model $5000$ epochs and select the best model according to the performance on the validation set. The parameters are optimized by Adam~\cite{kingma2014adam}. We run all experiments on a Lambda Labs instance with one A100 GPU.

\begin{table}[!htp]
\centering
\scriptsize
\setlength\tabcolsep{0.5pt}
\fontsize{6.5}{10}\selectfont  
\caption{Counterfactual outcomes estimation errors on two datasets. ``BC'' is the abbreviation of the BlogCatalog dataset. The errors are broken down in x-step future estimation ($x\in[1,2,3,4,5]$). MSE errors are reported. The best results are in boldface and the second best results are \underline{underlined}. $\model$-N is the variant of our model without any balancing; $\model$-T means balance only w.r.t. treatments; $\model$-I denotes balance only w.r.t. interference.}\label{tb:cf_in}
\vspace{-10pt}
\begin{tabular}{ccccccccc}\toprule[1.1pt]
Dataset &Model &1-step &2-step &3-step &4-step &5-step &Overall \\\midrule
\multirow{6}{*}{Flickr} 

&CDE &0.134\tiny{$\pm$0.015} &0.164\tiny{$\pm$0.017} &0.198\tiny{$\pm$0.021} &0.237\tiny{$\pm$0.023} &0.281\tiny{$\pm$0.026} &0.203\tiny{$\pm$0.205} \\
&GraphODE &0.237\tiny{$\pm$0.013} &0.276\tiny{$\pm$0.010} &0.313\tiny{$\pm$0.016} &0.347\tiny{$\pm$0.018} &0.379\tiny{$\pm$0.021} &0.310\tiny{$\pm$0.016} \\
&TE-CDE &0.189\tiny{$\pm$0.025} &0.216\tiny{$\pm$0.021} &0.246\tiny{$\pm$0.027} &0.281\tiny{$\pm$0.041} &0.326\tiny{$\pm$0.063} &0.252\tiny{$\pm$0.031} \\
\cmidrule(l{4pt}){2-8}
&\model-N  &0.089\tiny{$\pm$0.006} &0.102\tiny{$\pm$0.007} &0.114\tiny{$\pm$0.009} &0.126\tiny{$\pm$0.010} &0.139\tiny{$\pm$0.012} &0.114\tiny{$\pm$0.008} \\
&\model-T &\underline{0.058\tiny{$\pm$0.017}} &\underline{0.066\tiny{$\pm$0.020}} &\underline{0.075\tiny{$\pm$0.023}} &\underline{0.084\tiny{$\pm$0.027}} &\underline{0.098\tiny{$\pm$0.036}} &\underline{0.076\tiny{$\pm$0.025}} \\
&\model-I &0.069\tiny{$\pm$0.007} &0.080\tiny{$\pm$0.008} &0.091\tiny{$\pm$0.009} &0.103\tiny{$\pm$0.011} &0.115\tiny{$\pm$0.012} &0.092\tiny{$\pm$0.009} \\
&\model &\textbf{0.056\tiny{$\pm$0.009}} &\textbf{0.060\tiny{$\pm$0.009 }}&\textbf{0.067\tiny{$\pm$0.009}} &\textbf{0.070\tiny{$\pm$0.010}} &\textbf{0.077\tiny{$\pm$0.012}} &\textbf{0.065\tiny{$\pm$0.010}} \\\midrule[0.8pt]
\multirow{6}{*}{BC} 
&CDE &0.255\tiny{$\pm$0.120} &0.324\tiny{$\pm$0.178} &0.407\tiny{$\pm$0.263} &0.515\tiny{$\pm$0.383} &0.640\tiny{$\pm$0.549} &0.427\tiny{$\pm$0.296} \\
&GraphODE &0.195\tiny{$\pm$0.018} &0.223\tiny{$\pm$0.023} &0.251\tiny{$\pm$0.028} &0.280\tiny{$\pm$0.033} &0.309\tiny{$\pm$0.040} &0.252\tiny{$\pm$0.028} \\
&TE-CDE &0.316\tiny{$\pm$0.086} &0.351\tiny{$\pm$0.089} &0.399\tiny{$\pm$0.093} &0.493\tiny{$\pm$0.137} &0.725\tiny{$\pm$0.351} &0.457\tiny{$\pm$0.127} \\
\cmidrule(l{4pt}){2-8}
&\model-N &0.167\tiny{$\pm$0.012} &0.188\tiny{$\pm$0.015} &0.209\tiny{$\pm$0.018} &0.228\tiny{$\pm$0.023} &0.246\tiny{$\pm$0.028} &0.207\tiny{$\pm$0.019} \\
&\model-T &\textbf{0.139\tiny{$\pm$0.015}} &\textbf{0.154\tiny{$\pm$0.019}} &\textbf{0.172\tiny{$\pm$0.025}} &\textbf{0.189\tiny{$\pm$0.027}} &\textbf{0.202\tiny{$\pm$0.031}} &\textbf{0.171\tiny{$\pm$0.023}} \\
&\model-I &0.164\tiny{$\pm$0.016} &0.188\tiny{$\pm$0.020} &0.210\tiny{$\pm$0.024} &0.232\tiny{$\pm$0.029} &0.253\tiny{$\pm$0.035} &0.209\tiny{$\pm$0.025} \\
&\model &\underline{0.148\tiny{$\pm$0.015}} &\underline{0.166\tiny{$\pm$0.019}} &\underline{0.186\tiny{$\pm$0.023}} &\underline{0.205\tiny{$\pm$0.025}} &\underline{0.229\tiny{$\pm$0.029}} &\underline{0.186\tiny{$\pm$0.021}} \\
\bottomrule[1.1pt]
\end{tabular}
\vspace{-10pt}
\end{table}

\subsection{Can \model Deliver Accurate Estimations of Counterfactual Outcomes in Multi-Agent Dynamical Systems?}\label{sec:results}
We compare \model to three lines of models: 1) Continuous-time dynamical prediction models CDE and GraphODE. Note that these baselines are not causal models since they only preserve the dynamical statistical associations.  2) Continuous-time causal inference model TE-CDE. But it is not capable of capturing the mutual dependencies between units in multi-agent dynamical systems. 3) Variants of \model. We consider three variants: \model-N means there is no any balancing ($\alpha_{\mathbf{A}}=\alpha_{\mathbf{G}}=0$); \model-T denotes balancing only w.r.t. treatments ($\alpha_{\mathbf{A}}=1$, $\alpha_{\mathbf{G}}=0$); \model-I means balancing only w.r.t. interference ($\alpha_{\mathbf{A}}=0$, $\alpha_{\mathbf{G}}=1$). For a multi-agent dynamical system with $N$ nodes and length-$A$ treatment trajectories, the total number of possible treatments for all nodes is $O(A\cdot2^N)$. Therefore, it is intractable to enumerate all treatment combinations. To this end, we randomly flip $50\%$ of all observed treatments in each experiment. We estimate five-step (timestamp) ahead counterfactual outcomes and report estimation errors in Table.~\ref{tb:cf_in}. Generally, \model and the variants outperform the baselines by substantial margins. It is noteworthy that, despite being a causal model, TE-CDE performs clearly worse than the family of \model, because it ignores the mutual influence between units. This underscores our motivation to address this unique challenge in multi-agent dynamical systems. We also note \model-N is generally the weakest estimator among all variants, confirming the effectiveness of our proposed balancing objectives.

\begin{figure}[h]
 \centering
 \includegraphics[width=1\columnwidth]{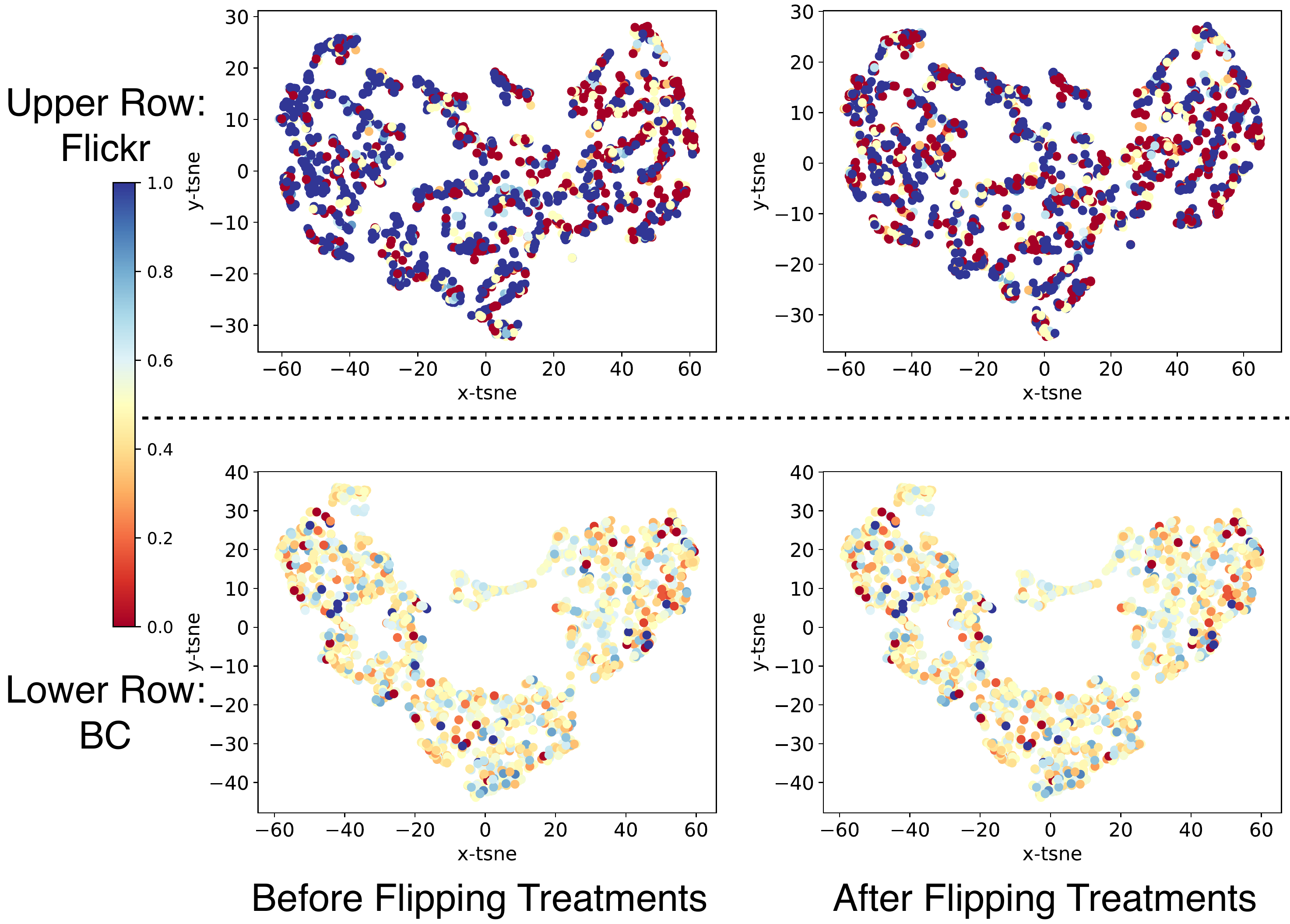}
  \caption{T-SNE projections of latent representations ``before'' (factual) and ``after'' (counterfactual) flipping the treatments. Each point represents a unit's latent representation. The points are colored by the units' corresponding interference. Upper row: Flickr dataset; Lower Row: BlogCatalog dataset.}\label{fig:why_bad}
\end{figure}
% \vspace{-0.5cm}

\textbf{Why Does \model-T Show Superior Performance Than \model on BlogCatalog Dataset? } On BlogCatalog dataset, we observe that balancing solely with respect to treatments (\model-T) yields the lowest estimation errors, even outperforming balancing both treatments and interference (\model). To understand this phenomenon, we project all units' latent representations $\mathbf{Z}^t$ into 2-D embeddings using T-SNE~\cite{van2008visualizing} and color these 2-D points by their corresponding interference in Fig.~\ref{fig:why_bad}. Specifically, we  compare the units' interference before and after flipping the treatments. Compared to Flickr, we notice that in BlogCatalog 1) the latent representations are already comparatively more balanced before flipping the treatments, and 2) the units' interference does not change significantly after flipping the treatments. This suggests that balancing solely with respect to treatments might be sufficient in the BlogCatalog dataset. Actually, since we use the same data simulation protocol for Flickr and BlogCatalog, this difference in interference distribution is expected to be caused by their distinct graph structures. Specifically, the average and standard derivation of node degrees of the two datasets are Flickr: $2.0\pm1.7$; BlogCatalog: $30.7\pm25.1$. Intuitively, the interference of high degrees nodes is more resistant to flipping a random portion of their neighbors, which is pretty common among nodes in BlogCatalog. We provide further breakdown studies to better understand how node degrees affect counterfactual outcomes estimation in Sec.~\ref{sec:degree_exp}.

\begin{figure}[h]
 \centering
 \includegraphics[width=1\columnwidth]{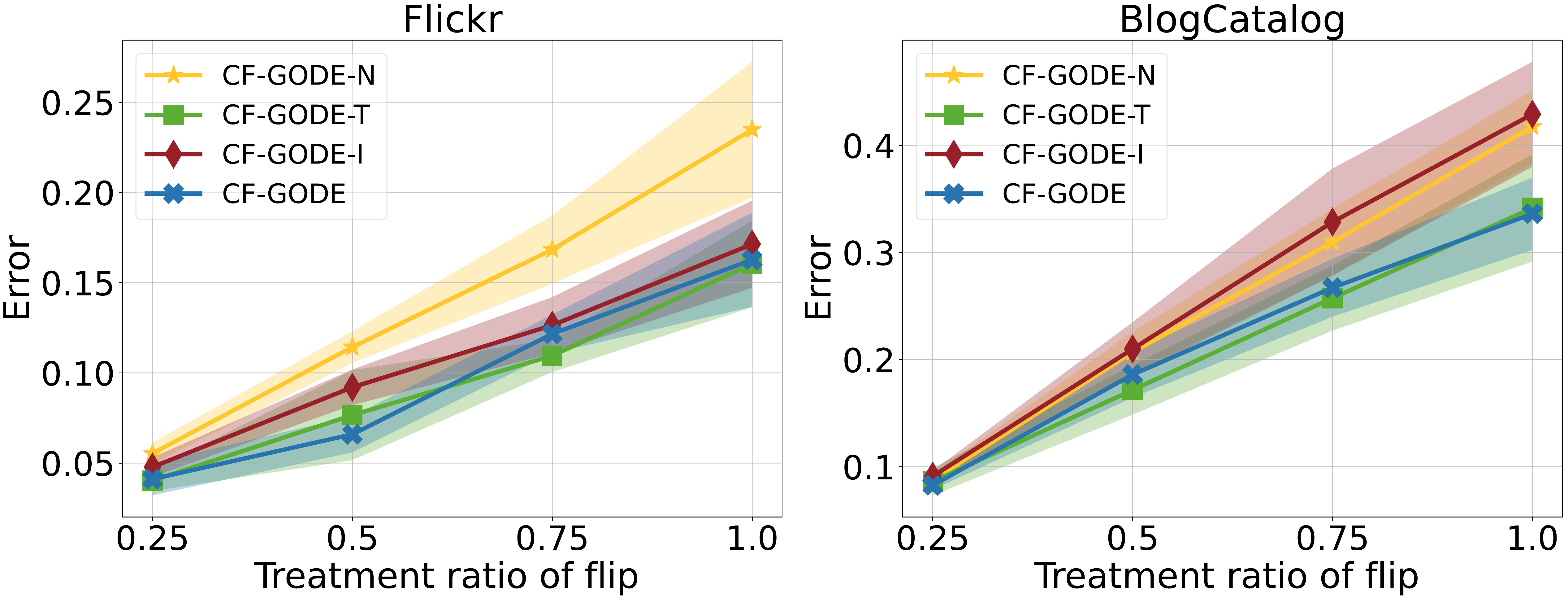}
  \caption{Counterfactual outcomes estimation errors w.r.t. the percentage of units in the graph whose treatments are flipped. Left: Flickr dataset; Right: BlogCatalog dataset.}\label{fig:flip}
\vspace{-15pt}
\end{figure}

\subsection{How Does \model Respond to The Flipping of Counterfactual Treatments?}

In the above experiments, the default treatment flipping ratio is set at $50\%$. It's intriguing to investigate how \model reacts to different flipping ratios, as this would indicate the degree of difference between factual and counterfactual outcomes in terms of treatments. To this end, we set the flip ratio as $[25\%,50\%,75\%,100\%]$, and present the results of \model and its variants under these settings in Fig.~\ref{fig:flip}. As expected, all models perform worse as the flip ratio increases, since the counterfactual treatments diverge further from the observed factual treatments. However, we observe that with balancing objectives, the error of \model increases generally slowly, highlighting the need for balancing objectives.

\begin{figure}[h]
 \centering
 \includegraphics[width=1\columnwidth]{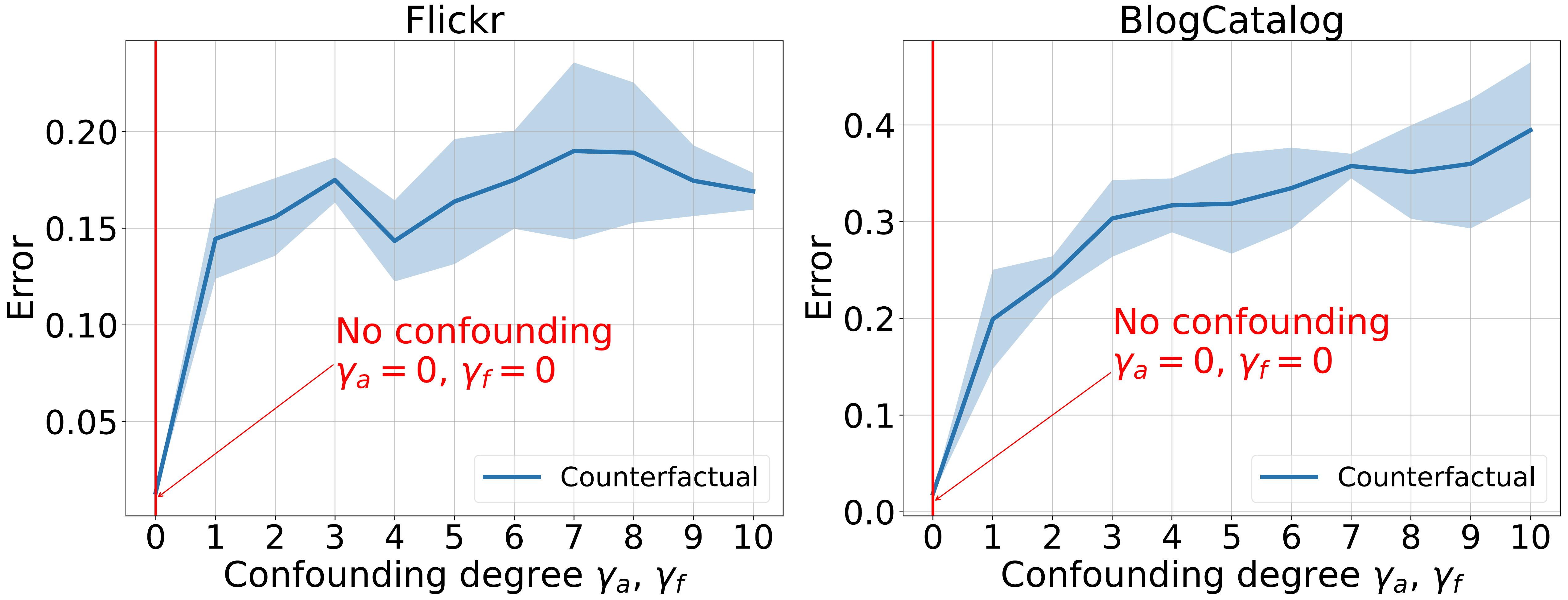}
  \caption{Counterfactual outcomes estimation errors w.r.t. 11 different confounding degrees $\gamma_a$ and $\gamma_f$. Note $\gamma_a$ and $\gamma_f$ are set as the same values in each experiment. The \redtext{red} line points to the ``no confounding bias'' setting ($\gamma_a = \gamma_f=0$).} \label{fig:confounding}
  \vspace{-10pt}
\end{figure}

\subsection{How Does \model Respond to Different Confounding Degrees?}\label{sec:degree}

In data simulation, we use coefficients $\gamma_a$, $\gamma_f$ to control the degree of the time-dependent and neighbor confounding bias. With larger values of these coefficients, the confounding bias is more severe, leading to increasingly imbalanced data. To study how \model works under varying confounding degrees, we set $\gamma_a=\gamma_f = \gamma$, where $\gamma\in[0,1,2,3,4,5,6,7,8,9,10]$, and present the counterfactual outcomes estimation errors of \model under these conditions in Fig.\ref{fig:confounding}. The errors increase as the confounding bias becomes more severe, but the rate of increase is relatively smooth, particularly on Flickr dataset. This implicates \model's robustness against high degree confounding bias. Additionally, \model produces low errors when there is no confounding bias ($\gamma_a=\gamma_f=0$), which demonstrates the compatibility of \model with such settings.

\subsection{How Does Graph Structure Impact Counterfactual Outcomes Estimation?}\label{sec:degree_exp}
As discussed in Sec.~\ref{sec:results}, the graph structure affects \model's performance on counterfactual outcomes estimation. To gain deeper insights into this phenomenon, we break down the estimation errors on BlogCatalog dataset according to node degrees in Table.~\ref{tb:degree}. Interestingly, we find that \model's counterfactual estimation errors decrease as the node degrees become higher. Intuitively, this might also be because the interference of high-degree nodes is more stable. However, in this paper, we do not have a theoretical understanding of the relationships between estimation errors and node degrees. We leave this line of research in future study.

\begin{table}[!htp]
\vspace{-10pt}
\centering
% \scriptsize
\fontsize{8}{10}\selectfont
\setlength{\tabcolsep}{12pt}
\caption{The breakdown of counterfactual outcomes estimation errors by units (nodes) degrees in BlogCatalog dataset.}\label{tb:degree}
\vspace{-10pt}
\begin{tabular}{ccccc}\toprule[1.1pt]
Degree &$\#$Nodes &Percentage$\%$ &Error \\\midrule
(0,5] &176 &10.2 &0.243\scriptsize{$\pm$0.337} \\
(5,10] &185 &10.6 &0.235\scriptsize{$\pm$0.344} \\
(10,20] &389 &22.5 &0.216\scriptsize{$\pm$0.296} \\
(20,30] &312 &18.0 &0.218\scriptsize{$\pm$0.314} \\
(20,30] &200 &11.5 &0.221\scriptsize{$\pm$0.295} \\
(30,40] &165 &9.5 &0.195\scriptsize{$\pm$0.284} \\
(40,50] &98 &5.7 &0.176\scriptsize{$\pm$0.269} \\
$>$50 &207 &12.0 &0.187\scriptsize{$\pm$0.261} \\
\bottomrule[1.1pt]
\end{tabular}
\end{table}
\vspace{-0.5cm}
\subsection{Can \model Be Generalized to New Multi-Agent Dynamical Systems?}
Standard counterfactual outcome estimations are typically conducted on units whose factual outcomes have been observed. However, the estimation of the potential outcomes on \emph{new} multi-agent dynamical systems is also of great importance. For instance, to predict the effects of an initial vaccine distribution strategy for a new community. To assess \model's ability to generalize to new systems, i.e., new graphs, we split the original graph into three subgraphs, denoted as training/validation/testing graphs (details in Appendix.~\ref{sec:exp_setting}). We train our model on the training graph and evaluate its potential outcome estimation on the testing graph. We report the results in Table.~\ref{tb:cf_out}. We note that the performance of \model and the variants on new graphs are also generally better than baselines, which is consistent with the estimation of the counterfactual outcomes within the same graph (Sec.~\ref{sec:results}). This demonstrates our model's generalizability to new multi-agent dynamical systems.

\begin{table}[!htp]
\centering
% \vspace{-0.5cm}
\scriptsize
\setlength\tabcolsep{0.5pt}
\fontsize{6.5}{10}\selectfont  
\caption{Generalization errors of potential outcomes prediction for new multi-agent dynamical systems (new graphs) on two datasets. ``BC'' is the abbreviation of the BlogCatalog dataset. The errors are broken down in x-step future estimation ($x\in[1,2,3,4,5]$). MSE errors are reported. The best results are in boldface and the second best results are \underline{underlined}. $\model$-N is the variant of our model without any balancing; $\model$-T means balance only w.r.t. treatments; $\model$-I denotes balance only w.r.t. interference.}\label{tb:cf_out}
\vspace{-10pt}
\begin{tabular}{ccccccccc}\toprule[1.1pt]
Dataset &Model &1-step &2-step &3-step &4-step &5-step &Overall \\\midrule
\multirow{7}{*}{Flickr} 
&CDE &0.134\tiny{$\pm$0.017} &0.166\tiny{$\pm$0.022} &0.201\tiny{$\pm$0.026} &0.241\tiny{$\pm$030} &0.285\tiny{$\pm$0.034} &0.205\tiny{$\pm$0.025} \\
&GraphODE &0.209\tiny{$\pm$0.010} &0.243\tiny{$\pm$0.012} &0.275\tiny{$\pm$0.015} &0.306\tiny{$\pm$0.018} &0.335\tiny{$\pm$0.022} &0.274\tiny{$\pm$0.014} \\
&TE-CDE &0.193\tiny{$\pm$0.023} &0.221\tiny{$\pm$0.021} &0.251\tiny{$\pm$0.027} &0.285\tiny{$\pm$0.040} &0.328\tiny{$\pm$0.061} &0.256\tiny{$\pm$0.030} \\
\cmidrule(l{4pt}){2-8}
&\model-N &0.087\tiny{$\pm$0.06} &0.099\tiny{$\pm$0.008} &0.111\tiny{$\pm$0.009} &0.122\tiny{$\pm$0.010} &0.134\tiny{$\pm$0.011} &0.111\tiny{$\pm$0.008} \\
&\model-T &\textbf{0.057\tiny{$\pm$0.014}} &\underline{0.064\tiny{$\pm$0.016}} &\underline{0.072\tiny{$\pm$0.018}} &\underline{0.081\tiny{$\pm$0.022}} &\underline{0.092\tiny{$\pm$0.029}} &\underline{0.738\tiny{$\pm$0.020}} \\
&\model-I &\underline{0.071\tiny{$\pm$0.007}} &0.083\tiny{$\pm$0.008} &0.096\tiny{$\pm$0.009} &0.109\tiny{$\pm$0.010} &0.122\tiny{$\pm$0.011} &0.096\tiny{$\pm$0.009} \\
&\model &\textbf{0.057\tiny{$\pm$0.008}} &\textbf{0.062\tiny{$\pm$0.009}} &\textbf{0.067\tiny{$\pm$0.010}} &\textbf{0.073\tiny{$\pm$0.011}} &\textbf{0.081\tiny{$\pm$0.013}} &\textbf{0.069\tiny{$\pm$0.010}} \\\midrule
\multirow{7}{*}{BC} 
&CDE &0.251\tiny{$\pm$0.113} &0.317\tiny{$\pm$0.174} &0.399\tiny{$\pm$0.259} &0.500\tiny{$\pm$0.380} &0.625\tiny{$\pm$0.548} &0.418\tiny{$\pm$0.294} \\
&GraphODE &0.183\tiny{$\pm$0.021} &0.210\tiny{$\pm$0.026} &0.237\tiny{$\pm$0.032} &0.265\tiny{$\pm$0.039} &0.293\tiny{$\pm$0.046} &0.237\tiny{$\pm$0.033} \\
&TE-CDE &0.328\tiny{$\pm$0.092} &0.372\tiny{$\pm$0.104} &0.440\tiny{$\pm$0.164} &0.582\tiny{$\pm$0.378} &0.933\tiny{$\pm$0.966} &0.457\tiny{$\pm$0.127} \\
\cmidrule(l{4pt}){2-8}
&\model-N &0.168\tiny{$\pm$0.008} &0.190\tiny{$\pm$0.009} &0.213\tiny{$\pm$0.012} &0.235\tiny{$\pm$0.015} &0.255\tiny{$\pm$0.021} &0.213\tiny{$\pm$0.013} \\
&\model-T &\textbf{0.141\tiny{$\pm$0.011}} &\textbf{0.157\tiny{$\pm$0.015}} &\textbf{0.175\tiny{$\pm$0.021}} &\textbf{0.192\tiny{$\pm$0.022}} &\textbf{0.203\tiny{$\pm$0.028}} &\textbf{0.174\tiny{$\pm$0.018}} \\
&\model-I &0.164\tiny{$\pm$0.014} &0.187\tiny{$\pm$0.018} &0.209\tiny{$\pm$0.022} &0.231\tiny{$\pm$0.028} &0.253\tiny{$\pm$0.034} &0.209\tiny{$\pm$0.023} \\
&\model &\underline{0.143\tiny{$\pm$0.013}} &\underline{0.162\tiny{$\pm$0.017}} &\underline{0.180\tiny{$\pm$0.022}} &\underline{0.199\tiny{$\pm$0.023}} &\underline{0.214\tiny{$\pm$0.025}} &\underline{0.180\tiny{$\pm$0.019}} \\
\bottomrule[1.1pt]
\end{tabular}
\end{table}

% \vspace{-20pt}

\begin{figure*}[h]
 \centering
 \includegraphics[width=2.\columnwidth]{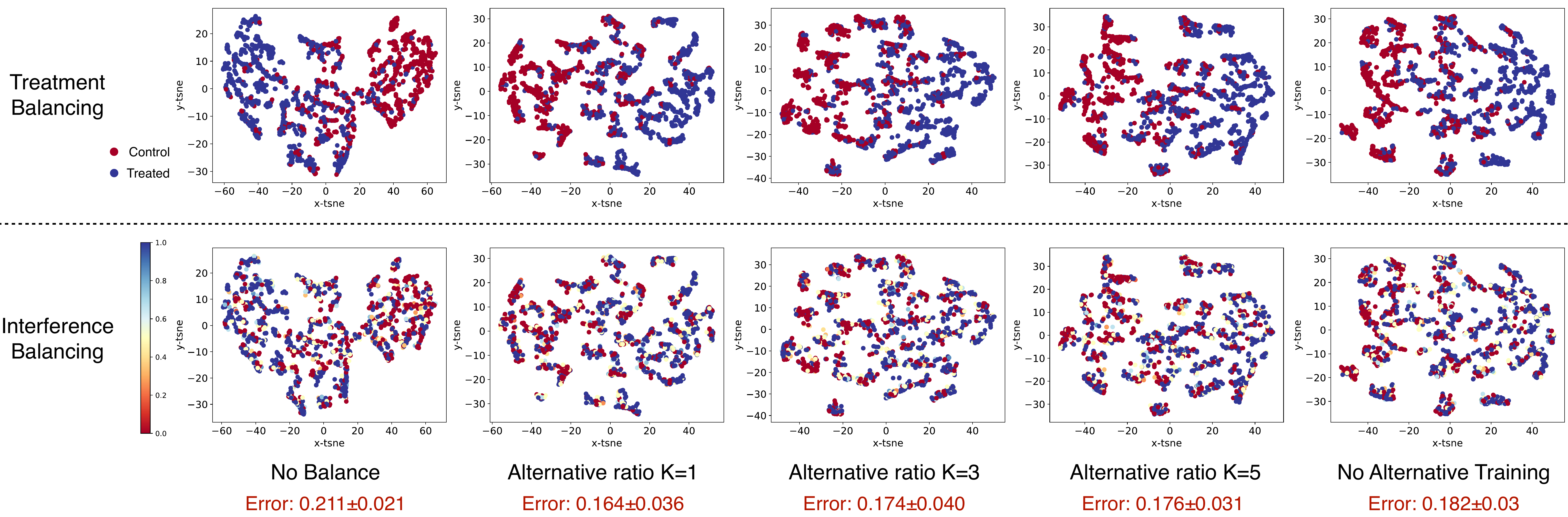}
  \caption{The T-SNE visualization of latent representations under different alternative training settings. Each point represents a unit's latent representation. The top line is colored by the observed treatments and the bottom line is colored by observed interference. Five columns from left to right: 1) only trained on $L^{\langle Y \rangle}$ (no balancing); 2) $\frac{Iter_{L}}{Iter_{L^{\langle Y \rangle}}} = 1$; 3) $\frac{Iter_{L}}{Iter_{L^{\langle Y \rangle}}} = 3$; 4) $\frac{Iter_{L}}{Iter_{L^{\langle Y \rangle}}} = 5$; 5) only trained on $L$ (no alternative training). Each setting's corresponding counterfactual estimation error is marked in \drtext{red}.    }\label{fig:embed}
\end{figure*}

\subsection{How Does Alternative Training Affect \model?}

\model uses an alternative training strategy to trade off the latent representation balancing and potential outcome prediction. We examine how this trade-off is performed under varying alternative ratios $K$, where $K=\frac{Iter_{L}}{Iter_{L^{\langle Y \rangle}}}$ represents the alternating training between the overall loss $L$ and the outcome prediction loss $L^{\langle Y \rangle}$. Fig.~\ref{fig:embed} shows the 2-D T-SNE projections of latent representations and their corresponding counterfactual estimation errors for different K values. We note that compared to solely training on $L^{\langle Y \rangle}$ (i.e., no balancing), training with $L$ is able to force the embeddings more balanced, suggesting the effectiveness of our proposed domain adversarial learning based balancing objectives. In addition, with a smaller $K$, \model achieves better estimation errors, while a bigger $K$ leads to more balanced latent representations. These results confirm that our alternative training is able to trade off between latent representation balancing and potential outcome prediction. In practice, choosing an appropriate value of $K$ is expected to be determined through empirical analysis for each dataset.

\begin{figure}[h]
 \centering
 \includegraphics[width=1\columnwidth]{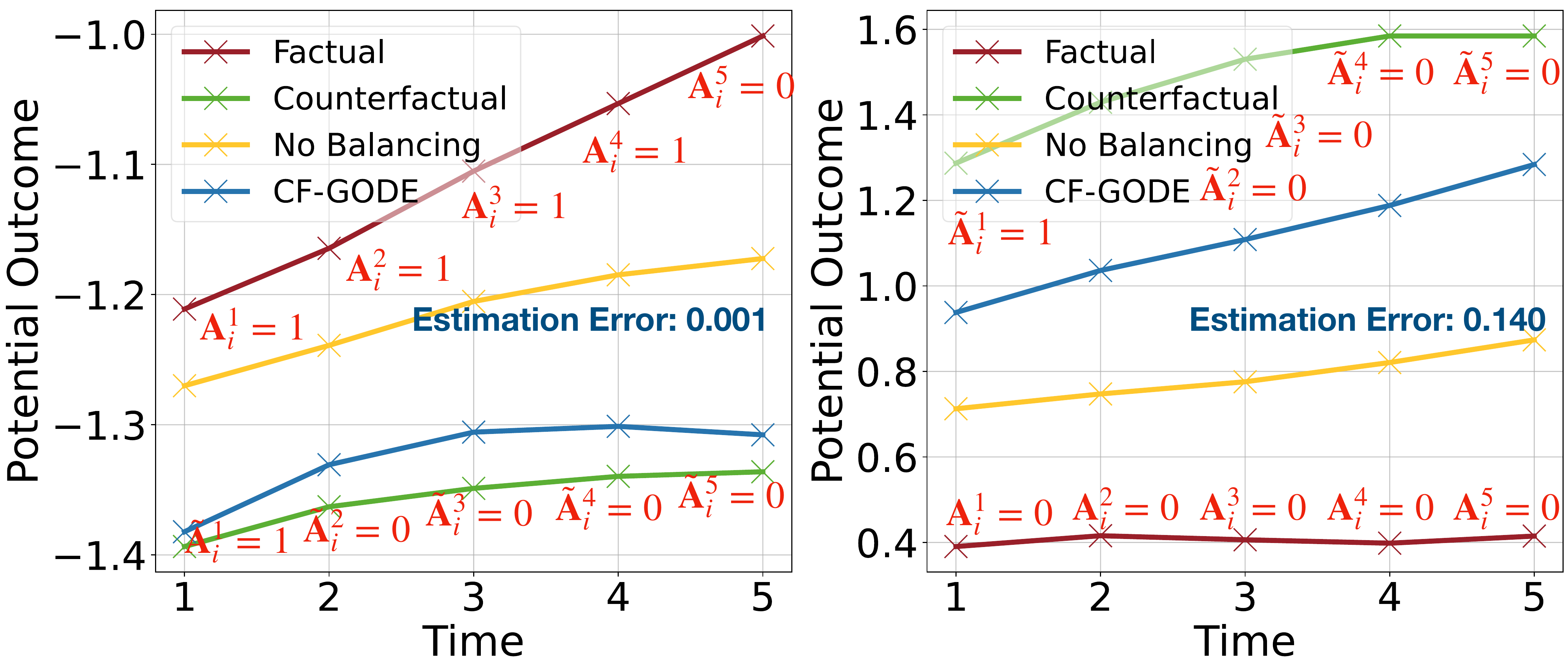}
  \caption{The predicted counterfactual outcomes of \model w/w.o balancing loss functions. The treatments of factual outcomes \redtext{$\mathbf{A}_i^t$} and treatments of counterfactual outcomes \redtext{$\mathbf{\Tilde{A}}_i^t$} are attached around the corresponding curves at each timestamp. The estimation errors are noted in \bluetext{blue}. Results are from Flickr dataset. Left: a successful case; right: a bad case.}\label{fig:case_study}
  \vspace{-15pt}
\end{figure}
% \vspace{-0.1cm}

\subsection{Case Study: When \model Is Good, and When It Is Not.}

To intuitively understand how \model works in estimating counterfactual outcomes and to study when \model would fail, we sample one successful unit and one failure unit from Flickr dataset. We draw their factual outcomes, counterfactual outcomes, and the estimations made by \model and \model-N (without balancing) in Fig.~\ref{fig:case_study}. In the successful case, the estimate by \model is able to conform to the counterfactual treatment trajectory, while \model-N still follows the factual trajectory. This shows the effectiveness of our proposed balancing objectives. However, \model also makes mistakes. In the failure case, its estimate fails to catch up with the counterfactual outcome trajectory, yielding a non-trivial error. We speculate that this is because the counterfactual outcome of this unit is quite distinct from the factual one in terms of data scale, making counterfactual estimation more difficult.

\subsection{How Hyperparamters Affect \model?}

The two balancing objectives are core to making \model causal. Therefore, we finally study the impact of their degrees, represented by $\alpha_{\mathbf{A}}$ and $\alpha_{\mathbf{G}}$ in the loss function, on the model performance. We test $\alpha_{\mathbf{A}}$ and $\alpha_{\mathbf{G}}$ values evenly ranging from $[0,0,1.0]$, and present the counterfactual outcomes estimation errors for each combination of $\alpha_{\mathbf{A}}$ and $\alpha_{\mathbf{G}}$ in Fig.~\ref{fig:params}.  Our results show that the errors are relatively higher when both $\alpha_{\mathbf{A}}$ and  $\alpha_{\mathbf{G}}$ are in low values, i.e., light balancing. On the other hand, with larger values, the estimation errors generally become lower, but with high variance. This indicates the effectiveness of the balancing objectives but also highlights the instability of domain adversarial learning based balancing.

\begin{figure}[h]
 \centering
 \includegraphics[width=1\columnwidth]{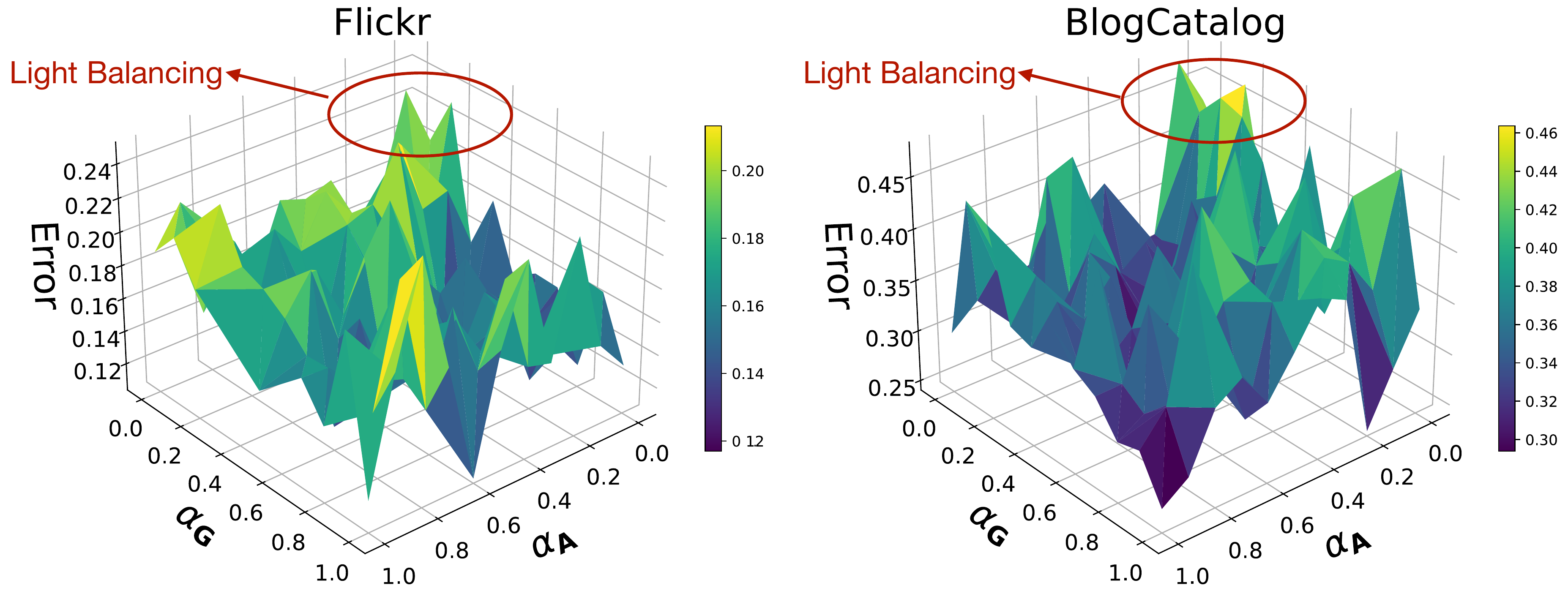}
  \caption{The counterfactual estimation errors w.r.t. different combinations of $\alpha_{\mathbf{A}}$ and  $\alpha_{\mathbf{G}}$. The ``Light Balancing'' settings where $\alpha_{\mathbf{A}}$ and $\alpha_{\mathbf{G}}$ are both in small values  are circles in \drtext{red}.  }\label{fig:params}
\end{figure}

\section{Conclusion}

In this paper, we study continuous-time counterfactual outcomes estimation in multi-agent dynamical systems, where units interact with each other. To this end, we propose \model, a novel causal model based on GraphODE to enable continuous potential outcomes prediction, and domain adversarial learning to remove confounding bias. We provide both theoretical justification and empirical analyses to demonstrate the effectiveness of our model. One limitation of \model is it needs the assumption of strong ignorability for multi-agent dynamical systems, which is not testable in practice. Recent studies relax this assumption by inferring latent proxy variables~\cite{wang2019blessings}, which could be a potential solution.
\begin{acks}
This work was partially supported by NSF 2211557, NSF 1937599, NSF 2119643, NSF 2303037,  NASA, SRC, Okawa Foundation Grant, Amazon Research Awards, Cisco research grant, Picsart Gifts, and Snapchat Gifts.
\end{acks}

%%
%% The next two lines define the bibliography style to be used, and
%% the bibliography file.
\clearpage
\bibliographystyle{ACM-Reference-Format}
\bibliography{reference}

%%% -*-BibTeX-*-
%%% Do NOT edit. File created by BibTeX with style
%%% ACM-Reference-Format-Journals [18-Jan-2012].

\begin{thebibliography}{58}

%%% ====================================================================
%%% NOTE TO THE USER: you can override these defaults by providing
%%% customized versions of any of these macros before the \bibliography
%%% command.  Each of them MUST provide its own final punctuation,
%%% except for \shownote{}, \showDOI{}, and \showURL{}.  The latter two
%%% do not use final punctuation, in order to avoid confusing it with
%%% the Web address.
%%%
%%% To suppress output of a particular field, define its macro to expand
%%% to an empty string, or better, \unskip, like this:
%%%
%%% \newcommand{\showDOI}[1]{\unskip}   % LaTeX syntax
%%%
%%% \def \showDOI #1{\unskip}           % plain TeX syntax
%%%
%%% ====================================================================

\ifx \showCODEN    \undefined \def \showCODEN     #1{\unskip}     \fi
\ifx \showDOI      \undefined \def \showDOI       #1{#1}\fi
\ifx \showISBNx    \undefined \def \showISBNx     #1{\unskip}     \fi
\ifx \showISBNxiii \undefined \def \showISBNxiii  #1{\unskip}     \fi
\ifx \showISSN     \undefined \def \showISSN      #1{\unskip}     \fi
\ifx \showLCCN     \undefined \def \showLCCN      #1{\unskip}     \fi
\ifx \shownote     \undefined \def \shownote      #1{#1}          \fi
\ifx \showarticletitle \undefined \def \showarticletitle #1{#1}   \fi
\ifx \showURL      \undefined \def \showURL       {\relax}        \fi
% The following commands are used for tagged output and should be
% invisible to TeX
\providecommand\bibfield[2]{#2}
\providecommand\bibinfo[2]{#2}
\providecommand\natexlab[1]{#1}
\providecommand\showeprint[2][]{arXiv:#2}

\bibitem[Arbour et~al\mbox{.}(2016)]%
        {arbour2016inferring}
\bibfield{author}{\bibinfo{person}{David Arbour}, \bibinfo{person}{Dan Garant},
  {and} \bibinfo{person}{David Jensen}.} \bibinfo{year}{2016}\natexlab{}.
\newblock \showarticletitle{Inferring network effects from observational data}.
  In \bibinfo{booktitle}{\emph{Proceedings of the 22nd ACM SIGKDD International
  Conference on Knowledge Discovery and Data Mining}}.
  \bibinfo{pages}{715--724}.
\newblock


\bibitem[Bellot and Van Der~Schaar(2021)]%
        {bellot2021policy}
\bibfield{author}{\bibinfo{person}{Alexis Bellot} {and}
  \bibinfo{person}{Mihaela Van Der~Schaar}.} \bibinfo{year}{2021}\natexlab{}.
\newblock \showarticletitle{Policy analysis using synthetic controls in
  continuous-time}. In \bibinfo{booktitle}{\emph{International Conference on
  Machine Learning}}. PMLR, \bibinfo{pages}{759--768}.
\newblock


\bibitem[Bica et~al\mbox{.}(2020)]%
        {BicaAJS20}
\bibfield{author}{\bibinfo{person}{Ioana Bica}, \bibinfo{person}{Ahmed~M.
  Alaa}, \bibinfo{person}{James Jordon}, {and} \bibinfo{person}{Mihaela van~der
  Schaar}.} \bibinfo{year}{2020}\natexlab{}.
\newblock \showarticletitle{Estimating counterfactual treatment outcomes over
  time through adversarially balanced representations}. In
  \bibinfo{booktitle}{\emph{8th International Conference on Learning
  Representations, {ICLR} 2020, Addis Ababa, Ethiopia, April 26-30, 2020}}.
\newblock


\bibitem[Blei et~al\mbox{.}(2003)]%
        {blei2003latent}
\bibfield{author}{\bibinfo{person}{David~M Blei}, \bibinfo{person}{Andrew~Y
  Ng}, {and} \bibinfo{person}{Michael~I Jordan}.}
  \bibinfo{year}{2003}\natexlab{}.
\newblock \showarticletitle{Latent dirichlet allocation}.
\newblock \bibinfo{journal}{\emph{Journal of machine Learning research}}
  \bibinfo{volume}{3}, \bibinfo{number}{Jan} (\bibinfo{year}{2003}),
  \bibinfo{pages}{993--1022}.
\newblock


\bibitem[Chen et~al\mbox{.}(2018)]%
        {chen2018neural}
\bibfield{author}{\bibinfo{person}{Ricky~TQ Chen}, \bibinfo{person}{Yulia
  Rubanova}, \bibinfo{person}{Jesse Bettencourt}, {and}
  \bibinfo{person}{David~K Duvenaud}.} \bibinfo{year}{2018}\natexlab{}.
\newblock \showarticletitle{Neural ordinary differential equations}.
\newblock \bibinfo{journal}{\emph{Advances in neural information processing
  systems}}  \bibinfo{volume}{31} (\bibinfo{year}{2018}).
\newblock


\bibitem[Chu et~al\mbox{.}(2021)]%
        {chu2021graph}
\bibfield{author}{\bibinfo{person}{Zhixuan Chu}, \bibinfo{person}{Stephen~L
  Rathbun}, {and} \bibinfo{person}{Sheng Li}.} \bibinfo{year}{2021}\natexlab{}.
\newblock \showarticletitle{Graph infomax adversarial learning for treatment
  effect estimation with networked observational data}. In
  \bibinfo{booktitle}{\emph{Proceedings of the 27th ACM SIGKDD Conference on
  Knowledge Discovery \& Data Mining}}. \bibinfo{pages}{176--184}.
\newblock


\bibitem[Cui et~al\mbox{.}(2022)]%
        {cui2022braingb}
\bibfield{author}{\bibinfo{person}{Hejie Cui}, \bibinfo{person}{Wei Dai},
  \bibinfo{person}{Yanqiao Zhu}, \bibinfo{person}{Xuan Kan},
  \bibinfo{person}{Antonio Aodong~Chen Gu}, \bibinfo{person}{Joshua Lukemire},
  \bibinfo{person}{Liang Zhan}, \bibinfo{person}{Lifang He},
  \bibinfo{person}{Ying Guo}, {and} \bibinfo{person}{Carl Yang}.}
  \bibinfo{year}{2022}\natexlab{}.
\newblock \showarticletitle{BrainGB: a benchmark for brain network analysis
  with graph neural networks}.
\newblock \bibinfo{journal}{\emph{IEEE Transactions on Medical Imaging}}
  (\bibinfo{year}{2022}).
\newblock


\bibitem[De~Brouwer et~al\mbox{.}(2022)]%
        {de2022predicting}
\bibfield{author}{\bibinfo{person}{Edward De~Brouwer}, \bibinfo{person}{Javier
  Gonzalez}, {and} \bibinfo{person}{Stephanie Hyland}.}
  \bibinfo{year}{2022}\natexlab{}.
\newblock \showarticletitle{Predicting the impact of treatments over time with
  uncertainty aware neural differential equations.}. In
  \bibinfo{booktitle}{\emph{International Conference on Artificial Intelligence
  and Statistics}}. PMLR, \bibinfo{pages}{4705--4722}.
\newblock


\bibitem[Durrant and McCammon(2011)]%
        {durrant2011molecular}
\bibfield{author}{\bibinfo{person}{Jacob~D Durrant} {and}
  \bibinfo{person}{J~Andrew McCammon}.} \bibinfo{year}{2011}\natexlab{}.
\newblock \showarticletitle{Molecular dynamics simulations and drug discovery}.
\newblock \bibinfo{journal}{\emph{BMC biology}} \bibinfo{volume}{9},
  \bibinfo{number}{1} (\bibinfo{year}{2011}), \bibinfo{pages}{1--9}.
\newblock


\bibitem[Forastiere et~al\mbox{.}(2021)]%
        {forastiere2021identification}
\bibfield{author}{\bibinfo{person}{Laura Forastiere},
  \bibinfo{person}{Edoardo~M Airoldi}, {and} \bibinfo{person}{Fabrizia
  Mealli}.} \bibinfo{year}{2021}\natexlab{}.
\newblock \showarticletitle{Identification and estimation of treatment and
  interference effects in observational studies on networks}.
\newblock \bibinfo{journal}{\emph{J. Amer. Statist. Assoc.}}
  \bibinfo{volume}{116}, \bibinfo{number}{534} (\bibinfo{year}{2021}),
  \bibinfo{pages}{901--918}.
\newblock


\bibitem[Fujii et~al\mbox{.}(2022)]%
        {fujii2022estimating}
\bibfield{author}{\bibinfo{person}{Keisuke Fujii}, \bibinfo{person}{Koh
  Takeuchi}, \bibinfo{person}{Atsushi Kuribayashi}, \bibinfo{person}{Naoya
  Takeishi}, \bibinfo{person}{Yoshinobu Kawahara}, {and}
  \bibinfo{person}{Kazuya Takeda}.} \bibinfo{year}{2022}\natexlab{}.
\newblock \showarticletitle{Estimating counterfactual treatment outcomes over
  time in multi-vehicle simulation}. In \bibinfo{booktitle}{\emph{Proceedings
  of the 30th International Conference on Advances in Geographic Information
  Systems}}. \bibinfo{pages}{1--4}.
\newblock


\bibitem[Ganin et~al\mbox{.}(2016)]%
        {ganin2016domain}
\bibfield{author}{\bibinfo{person}{Yaroslav Ganin}, \bibinfo{person}{Evgeniya
  Ustinova}, \bibinfo{person}{Hana Ajakan}, \bibinfo{person}{Pascal Germain},
  \bibinfo{person}{Hugo Larochelle}, \bibinfo{person}{Fran{\c{c}}ois
  Laviolette}, \bibinfo{person}{Mario Marchand}, {and} \bibinfo{person}{Victor
  Lempitsky}.} \bibinfo{year}{2016}\natexlab{}.
\newblock \showarticletitle{Domain-adversarial training of neural networks}.
\newblock \bibinfo{journal}{\emph{The journal of machine learning research}}
  \bibinfo{volume}{17}, \bibinfo{number}{1} (\bibinfo{year}{2016}),
  \bibinfo{pages}{2096--2030}.
\newblock


\bibitem[Gazi and Fidan(2007)]%
        {gazi2007coordination}
\bibfield{author}{\bibinfo{person}{Veysel Gazi} {and}
  \bibinfo{person}{Bar{\i}{\c{s}} Fidan}.} \bibinfo{year}{2007}\natexlab{}.
\newblock \showarticletitle{Coordination and control of multi-agent dynamic
  systems: Models and approaches}. In \bibinfo{booktitle}{\emph{Swarm Robotics:
  Second International Workshop, SAB 2006, Rome, Italy, September 30-October 1,
  2006, Revised Selected Papers 2}}. Springer, \bibinfo{pages}{71--102}.
\newblock


\bibitem[Geng et~al\mbox{.}(2017)]%
        {geng2017prediction}
\bibfield{author}{\bibinfo{person}{Changran Geng}, \bibinfo{person}{Harald
  Paganetti}, {and} \bibinfo{person}{Clemens Grassberger}.}
  \bibinfo{year}{2017}\natexlab{}.
\newblock \showarticletitle{Prediction of treatment response for combined
  chemo-and radiation therapy for non-small cell lung cancer patients using a
  bio-mathematical model}.
\newblock \bibinfo{journal}{\emph{Scientific reports}} \bibinfo{volume}{7},
  \bibinfo{number}{1} (\bibinfo{year}{2017}), \bibinfo{pages}{13542}.
\newblock


\bibitem[Goutelle et~al\mbox{.}(2008)]%
        {goutelle2008hill}
\bibfield{author}{\bibinfo{person}{Sylvain Goutelle}, \bibinfo{person}{Michel
  Maurin}, \bibinfo{person}{Florent Rougier}, \bibinfo{person}{Xavier Barbaut},
  \bibinfo{person}{Laurent Bourguignon}, \bibinfo{person}{Michel Ducher}, {and}
  \bibinfo{person}{Pascal Maire}.} \bibinfo{year}{2008}\natexlab{}.
\newblock \showarticletitle{The Hill equation: a review of its capabilities in
  pharmacological modelling}.
\newblock \bibinfo{journal}{\emph{Fundamental \& clinical pharmacology}}
  \bibinfo{volume}{22}, \bibinfo{number}{6} (\bibinfo{year}{2008}),
  \bibinfo{pages}{633--648}.
\newblock


\bibitem[Guo et~al\mbox{.}(2020)]%
        {guo2020learning}
\bibfield{author}{\bibinfo{person}{Ruocheng Guo}, \bibinfo{person}{Jundong Li},
  {and} \bibinfo{person}{Huan Liu}.} \bibinfo{year}{2020}\natexlab{}.
\newblock \showarticletitle{Learning individual causal effects from networked
  observational data}. In \bibinfo{booktitle}{\emph{Proceedings of the 13th
  International Conference on Web Search and Data Mining}}.
  \bibinfo{pages}{232--240}.
\newblock


\bibitem[Gwak et~al\mbox{.}(2020)]%
        {gwak2020neural}
\bibfield{author}{\bibinfo{person}{Daehoon Gwak}, \bibinfo{person}{Gyuhyeon
  Sim}, \bibinfo{person}{Michael Poli}, \bibinfo{person}{Stefano Massaroli},
  \bibinfo{person}{Jaegul Choo}, {and} \bibinfo{person}{Edward Choi}.}
  \bibinfo{year}{2020}\natexlab{}.
\newblock \showarticletitle{Neural ordinary differential equations for
  intervention modeling}.
\newblock \bibinfo{journal}{\emph{arXiv preprint arXiv:2010.08304}}
  (\bibinfo{year}{2020}).
\newblock


\bibitem[Halloran and Hudgens(2012)]%
        {halloran2012causal}
\bibfield{author}{\bibinfo{person}{M~Elizabeth Halloran} {and}
  \bibinfo{person}{Michael~G Hudgens}.} \bibinfo{year}{2012}\natexlab{}.
\newblock \showarticletitle{Causal inference for vaccine effects on
  infectiousness}.
\newblock \bibinfo{journal}{\emph{The international journal of biostatistics}}
  (\bibinfo{year}{2012}).
\newblock


\bibitem[Huang et~al\mbox{.}(2020)]%
        {huang2020learning}
\bibfield{author}{\bibinfo{person}{Zijie Huang}, \bibinfo{person}{Yizhou Sun},
  {and} \bibinfo{person}{Wei Wang}.} \bibinfo{year}{2020}\natexlab{}.
\newblock \showarticletitle{Learning continuous system dynamics from
  irregularly-sampled partial observations}.
\newblock \bibinfo{journal}{\emph{Advances in Neural Information Processing
  Systems}}  \bibinfo{volume}{33} (\bibinfo{year}{2020}),
  \bibinfo{pages}{16177--16187}.
\newblock


\bibitem[Huang et~al\mbox{.}(2021)]%
        {huang2021coupled}
\bibfield{author}{\bibinfo{person}{Zijie Huang}, \bibinfo{person}{Yizhou Sun},
  {and} \bibinfo{person}{Wei Wang}.} \bibinfo{year}{2021}\natexlab{}.
\newblock \showarticletitle{Coupled Graph ODE for Learning Interacting System
  Dynamics.}. In \bibinfo{booktitle}{\emph{KDD}}. \bibinfo{pages}{705--715}.
\newblock


\bibitem[Jiang and Sun(2022)]%
        {jiang2022estimating}
\bibfield{author}{\bibinfo{person}{Song Jiang} {and} \bibinfo{person}{Yizhou
  Sun}.} \bibinfo{year}{2022}\natexlab{}.
\newblock \showarticletitle{Estimating Causal Effects on Networked
  Observational Data via Representation Learning}. In
  \bibinfo{booktitle}{\emph{Proceedings of the 31st ACM International
  Conference on Information \& Knowledge Management}}.
  \bibinfo{pages}{852--861}.
\newblock


\bibitem[Johansson et~al\mbox{.}(2016)]%
        {johansson2016learning}
\bibfield{author}{\bibinfo{person}{Fredrik Johansson}, \bibinfo{person}{Uri
  Shalit}, {and} \bibinfo{person}{David Sontag}.}
  \bibinfo{year}{2016}\natexlab{}.
\newblock \showarticletitle{Learning representations for counterfactual
  inference}. In \bibinfo{booktitle}{\emph{International conference on machine
  learning}}. PMLR, \bibinfo{pages}{3020--3029}.
\newblock


\bibitem[Karypis and Kumar(1998)]%
        {karypis1998fast}
\bibfield{author}{\bibinfo{person}{George Karypis} {and} \bibinfo{person}{Vipin
  Kumar}.} \bibinfo{year}{1998}\natexlab{}.
\newblock \showarticletitle{A fast and high quality multilevel scheme for
  partitioning irregular graphs}.
\newblock \bibinfo{journal}{\emph{SIAM Journal on scientific Computing}}
  \bibinfo{volume}{20}, \bibinfo{number}{1} (\bibinfo{year}{1998}),
  \bibinfo{pages}{359--392}.
\newblock


\bibitem[Kidger et~al\mbox{.}(2020)]%
        {kidger2020neural}
\bibfield{author}{\bibinfo{person}{Patrick Kidger}, \bibinfo{person}{James
  Morrill}, \bibinfo{person}{James Foster}, {and} \bibinfo{person}{Terry
  Lyons}.} \bibinfo{year}{2020}\natexlab{}.
\newblock \showarticletitle{Neural controlled differential equations for
  irregular time series}.
\newblock \bibinfo{journal}{\emph{Advances in Neural Information Processing
  Systems}}  \bibinfo{volume}{33} (\bibinfo{year}{2020}),
  \bibinfo{pages}{6696--6707}.
\newblock


\bibitem[Kingma and Ba(2015)]%
        {kingma2014adam}
\bibfield{author}{\bibinfo{person}{Diederik~P. Kingma} {and}
  \bibinfo{person}{Jimmy Ba}.} \bibinfo{year}{2015}\natexlab{}.
\newblock \showarticletitle{Adam: {A} Method for Stochastic Optimization}. In
  \bibinfo{booktitle}{\emph{3rd International Conference on Learning
  Representations, {ICLR} 2015, San Diego, CA, USA, May 7-9, 2015, Conference
  Track Proceedings}}.
\newblock


\bibitem[Kipf and Welling(2017)]%
        {KipfW17}
\bibfield{author}{\bibinfo{person}{Thomas~N. Kipf} {and} \bibinfo{person}{Max
  Welling}.} \bibinfo{year}{2017}\natexlab{}.
\newblock \showarticletitle{Semi-Supervised Classification with Graph
  Convolutional Networks}. In \bibinfo{booktitle}{\emph{5th International
  Conference on Learning Representations, {ICLR} 2017, Toulon, France, April
  24-26, 2017, Conference Track Proceedings}}.
\newblock


\bibitem[Lim et~al\mbox{.}(2018)]%
        {lim2018forecasting}
\bibfield{author}{\bibinfo{person}{Bryan Lim}, \bibinfo{person}{Ahmed Alaa},
  {and} \bibinfo{person}{Mihaela Van Der~Schaar}.}
  \bibinfo{year}{2018}\natexlab{}.
\newblock \showarticletitle{Forecasting treatment responses over time using
  recurrent marginal structural networks}.
\newblock \bibinfo{journal}{\emph{advances in neural information processing
  systems}}  \bibinfo{volume}{31} (\bibinfo{year}{2018}).
\newblock


\bibitem[Ma et~al\mbox{.}(2021)]%
        {ma2021deconfounding}
\bibfield{author}{\bibinfo{person}{Jing Ma}, \bibinfo{person}{Ruocheng Guo},
  \bibinfo{person}{Chen Chen}, \bibinfo{person}{Aidong Zhang}, {and}
  \bibinfo{person}{Jundong Li}.} \bibinfo{year}{2021}\natexlab{}.
\newblock \showarticletitle{Deconfounding with networked observational data in
  a dynamic environment}. In \bibinfo{booktitle}{\emph{Proceedings of the 14th
  ACM International Conference on Web Search and Data Mining}}.
  \bibinfo{pages}{166--174}.
\newblock


\bibitem[Ma et~al\mbox{.}(2022)]%
        {ma2022learning}
\bibfield{author}{\bibinfo{person}{Jing Ma}, \bibinfo{person}{Mengting Wan},
  \bibinfo{person}{Longqi Yang}, \bibinfo{person}{Jundong Li},
  \bibinfo{person}{Brent Hecht}, {and} \bibinfo{person}{Jaime Teevan}.}
  \bibinfo{year}{2022}\natexlab{}.
\newblock \showarticletitle{Learning causal effects on hypergraphs}. In
  \bibinfo{booktitle}{\emph{Proceedings of the 28th ACM SIGKDD Conference on
  Knowledge Discovery and Data Mining}}. \bibinfo{pages}{1202--1212}.
\newblock


\bibitem[Ma and Tresp(2021)]%
        {ma2021causal}
\bibfield{author}{\bibinfo{person}{Yunpu Ma} {and} \bibinfo{person}{Volker
  Tresp}.} \bibinfo{year}{2021}\natexlab{}.
\newblock \showarticletitle{Causal inference under networked interference and
  intervention policy enhancement}. In \bibinfo{booktitle}{\emph{International
  Conference on Artificial Intelligence and Statistics}}. PMLR,
  \bibinfo{pages}{3700--3708}.
\newblock


\bibitem[Massaroli et~al\mbox{.}(2020)]%
        {massaroli2020dissecting}
\bibfield{author}{\bibinfo{person}{Stefano Massaroli}, \bibinfo{person}{Michael
  Poli}, \bibinfo{person}{Jinkyoo Park}, \bibinfo{person}{Atsushi Yamashita},
  {and} \bibinfo{person}{Hajime Asama}.} \bibinfo{year}{2020}\natexlab{}.
\newblock \showarticletitle{Dissecting neural odes}.
\newblock \bibinfo{journal}{\emph{Advances in Neural Information Processing
  Systems}}  \bibinfo{volume}{33} (\bibinfo{year}{2020}),
  \bibinfo{pages}{3952--3963}.
\newblock


\bibitem[Medlock and Galvani(2009)]%
        {medlock2009optimizing}
\bibfield{author}{\bibinfo{person}{Jan Medlock} {and} \bibinfo{person}{Alison~P
  Galvani}.} \bibinfo{year}{2009}\natexlab{}.
\newblock \showarticletitle{Optimizing influenza vaccine distribution}.
\newblock \bibinfo{journal}{\emph{Science}} \bibinfo{volume}{325},
  \bibinfo{number}{5948} (\bibinfo{year}{2009}), \bibinfo{pages}{1705--1708}.
\newblock


\bibitem[Melnychuk et~al\mbox{.}(2022)]%
        {MelnychukFF22}
\bibfield{author}{\bibinfo{person}{Valentyn Melnychuk}, \bibinfo{person}{Dennis
  Frauen}, {and} \bibinfo{person}{Stefan Feuerriegel}.}
  \bibinfo{year}{2022}\natexlab{}.
\newblock \showarticletitle{Causal Transformer for Estimating Counterfactual
  Outcomes}. In \bibinfo{booktitle}{\emph{International Conference on Machine
  Learning, {ICML} 2022, 17-23 July 2022, Baltimore, Maryland, {USA}}}
  \emph{(\bibinfo{series}{Proceedings of Machine Learning Research},
  Vol.~\bibinfo{volume}{162})}. \bibinfo{publisher}{{PMLR}},
  \bibinfo{pages}{15293--15329}.
\newblock


\bibitem[Pearl(2009)]%
        {pearl2009causality}
\bibfield{author}{\bibinfo{person}{Judea Pearl}.}
  \bibinfo{year}{2009}\natexlab{}.
\newblock \bibinfo{booktitle}{\emph{Causality}}.
\newblock \bibinfo{publisher}{Cambridge university press}.
\newblock


\bibitem[Platt et~al\mbox{.}(2009)]%
        {platt2009time}
\bibfield{author}{\bibinfo{person}{Robert~W Platt}, \bibinfo{person}{Enrique~F
  Schisterman}, {and} \bibinfo{person}{Stephen~R Cole}.}
  \bibinfo{year}{2009}\natexlab{}.
\newblock \showarticletitle{Time-modified confounding}.
\newblock \bibinfo{journal}{\emph{American journal of epidemiology}}
  \bibinfo{volume}{170}, \bibinfo{number}{6} (\bibinfo{year}{2009}),
  \bibinfo{pages}{687--694}.
\newblock


\bibitem[Poli et~al\mbox{.}(2019)]%
        {poli2019graph}
\bibfield{author}{\bibinfo{person}{Michael Poli}, \bibinfo{person}{Stefano
  Massaroli}, \bibinfo{person}{Junyoung Park}, \bibinfo{person}{Atsushi
  Yamashita}, \bibinfo{person}{Hajime Asama}, {and} \bibinfo{person}{Jinkyoo
  Park}.} \bibinfo{year}{2019}\natexlab{}.
\newblock \showarticletitle{Graph neural ordinary differential equations}.
\newblock \bibinfo{journal}{\emph{arXiv preprint arXiv:1911.07532}}
  (\bibinfo{year}{2019}).
\newblock


\bibitem[Porter and Gleeson(2014)]%
        {porter2014dynamical}
\bibfield{author}{\bibinfo{person}{Mason~A Porter} {and}
  \bibinfo{person}{James~P Gleeson}.} \bibinfo{year}{2014}\natexlab{}.
\newblock \showarticletitle{Dynamical systems on networks: A tutorial}.
\newblock \bibinfo{journal}{\emph{arXiv preprint arXiv:1403.7663}}
  (\bibinfo{year}{2014}).
\newblock


\bibitem[Qian et~al\mbox{.}(2021)]%
        {qian2021integrating}
\bibfield{author}{\bibinfo{person}{Zhaozhi Qian}, \bibinfo{person}{William
  Zame}, \bibinfo{person}{Lucas Fleuren}, \bibinfo{person}{Paul Elbers}, {and}
  \bibinfo{person}{Mihaela van~der Schaar}.} \bibinfo{year}{2021}\natexlab{}.
\newblock \showarticletitle{Integrating expert ODEs into Neural ODEs:
  Pharmacology and disease progression}.
\newblock \bibinfo{journal}{\emph{Advances in Neural Information Processing
  Systems}} (\bibinfo{year}{2021}).
\newblock


\bibitem[Robins and Hern{\'a}n(2009)]%
        {robins2009estimation}
\bibfield{author}{\bibinfo{person}{James~M Robins} {and}
  \bibinfo{person}{Miguel~A Hern{\'a}n}.} \bibinfo{year}{2009}\natexlab{}.
\newblock \showarticletitle{Estimation of the causal effects of time-varying
  exposures}.
\newblock \bibinfo{journal}{\emph{Longitudinal data analysis}}
  \bibinfo{volume}{553} (\bibinfo{year}{2009}), \bibinfo{pages}{599}.
\newblock


\bibitem[Robins et~al\mbox{.}(2000)]%
        {robins2000marginal}
\bibfield{author}{\bibinfo{person}{James~M Robins},
  \bibinfo{person}{Miguel~Angel Hernan}, {and} \bibinfo{person}{Babette
  Brumback}.} \bibinfo{year}{2000}\natexlab{}.
\newblock \bibinfo{title}{Marginal structural models and causal inference in
  epidemiology}.
\newblock , \bibinfo{numpages}{550--560}~pages.
\newblock


\bibitem[Rosenbaum(1987)]%
        {rosenbaum1987model}
\bibfield{author}{\bibinfo{person}{Paul~R Rosenbaum}.}
  \bibinfo{year}{1987}\natexlab{}.
\newblock \showarticletitle{Model-based direct adjustment}.
\newblock \bibinfo{journal}{\emph{Journal of the American statistical
  Association}} \bibinfo{volume}{82}, \bibinfo{number}{398}
  (\bibinfo{year}{1987}), \bibinfo{pages}{387--394}.
\newblock


\bibitem[Rosenbaum and Rubin(1983)]%
        {rosenbaum1983central}
\bibfield{author}{\bibinfo{person}{Paul~R Rosenbaum} {and}
  \bibinfo{person}{Donald~B Rubin}.} \bibinfo{year}{1983}\natexlab{}.
\newblock \showarticletitle{The central role of the propensity score in
  observational studies for causal effects}.
\newblock \bibinfo{journal}{\emph{Biometrika}} \bibinfo{volume}{70},
  \bibinfo{number}{1} (\bibinfo{year}{1983}), \bibinfo{pages}{41--55}.
\newblock


\bibitem[Rubanova et~al\mbox{.}(2019)]%
        {RubanovaCD19}
\bibfield{author}{\bibinfo{person}{Yulia Rubanova}, \bibinfo{person}{Tian~Qi
  Chen}, {and} \bibinfo{person}{David Duvenaud}.}
  \bibinfo{year}{2019}\natexlab{}.
\newblock \showarticletitle{Latent Ordinary Differential Equations for
  Irregularly-Sampled Time Series}. In \bibinfo{booktitle}{\emph{Advances in
  Neural Information Processing Systems}}.
\newblock


\bibitem[Rubin(1978)]%
        {rubin1978bayesian}
\bibfield{author}{\bibinfo{person}{Donald~B Rubin}.}
  \bibinfo{year}{1978}\natexlab{}.
\newblock \showarticletitle{Bayesian inference for causal effects: The role of
  randomization}.
\newblock \bibinfo{journal}{\emph{The Annals of statistics}}
  (\bibinfo{year}{1978}), \bibinfo{pages}{34--58}.
\newblock


\bibitem[Schober et~al\mbox{.}(2019)]%
        {schober2019probabilistic}
\bibfield{author}{\bibinfo{person}{Michael Schober}, \bibinfo{person}{Simo
  S{\"a}rkk{\"a}}, {and} \bibinfo{person}{Philipp Hennig}.}
  \bibinfo{year}{2019}\natexlab{}.
\newblock \showarticletitle{A probabilistic model for the numerical solution of
  initial value problems}.
\newblock \bibinfo{journal}{\emph{Statistics and Computing}}
  \bibinfo{volume}{29}, \bibinfo{number}{1} (\bibinfo{year}{2019}),
  \bibinfo{pages}{99--122}.
\newblock


\bibitem[Seedat et~al\mbox{.}(2022)]%
        {SeedatIBQS22}
\bibfield{author}{\bibinfo{person}{Nabeel Seedat}, \bibinfo{person}{Fergus
  Imrie}, \bibinfo{person}{Alexis Bellot}, \bibinfo{person}{Zhaozhi Qian},
  {and} \bibinfo{person}{Mihaela van~der Schaar}.}
  \bibinfo{year}{2022}\natexlab{}.
\newblock \showarticletitle{Continuous-Time Modeling of Counterfactual Outcomes
  Using Neural Controlled Differential Equations}. In
  \bibinfo{booktitle}{\emph{International Conference on Machine Learning,
  {ICML} 2022, 17-23 July 2022, Baltimore, Maryland, {USA}}}
  \emph{(\bibinfo{series}{Proceedings of Machine Learning Research},
  Vol.~\bibinfo{volume}{162})}. \bibinfo{publisher}{{PMLR}},
  \bibinfo{pages}{19497--19521}.
\newblock


\bibitem[Shalit et~al\mbox{.}(2017)]%
        {shalit2017estimating}
\bibfield{author}{\bibinfo{person}{Uri Shalit}, \bibinfo{person}{Fredrik~D
  Johansson}, {and} \bibinfo{person}{David Sontag}.}
  \bibinfo{year}{2017}\natexlab{}.
\newblock \showarticletitle{Estimating individual treatment effect:
  generalization bounds and algorithms}. In
  \bibinfo{booktitle}{\emph{International Conference on Machine Learning}}.
  PMLR, \bibinfo{pages}{3076--3085}.
\newblock


\bibitem[Van~der Maaten and Hinton(2008)]%
        {van2008visualizing}
\bibfield{author}{\bibinfo{person}{Laurens Van~der Maaten} {and}
  \bibinfo{person}{Geoffrey Hinton}.} \bibinfo{year}{2008}\natexlab{}.
\newblock \showarticletitle{Visualizing data using t-SNE.}
\newblock \bibinfo{journal}{\emph{Journal of machine learning research}}
  \bibinfo{volume}{9}, \bibinfo{number}{11} (\bibinfo{year}{2008}).
\newblock


\bibitem[Vaswani et~al\mbox{.}(2017)]%
        {vaswani2017attention}
\bibfield{author}{\bibinfo{person}{Ashish Vaswani}, \bibinfo{person}{Noam
  Shazeer}, \bibinfo{person}{Niki Parmar}, \bibinfo{person}{Jakob Uszkoreit},
  \bibinfo{person}{Llion Jones}, \bibinfo{person}{Aidan~N Gomez},
  \bibinfo{person}{{\L}ukasz Kaiser}, {and} \bibinfo{person}{Illia
  Polosukhin}.} \bibinfo{year}{2017}\natexlab{}.
\newblock \showarticletitle{Attention is all you need}.
\newblock \bibinfo{journal}{\emph{Advances in neural information processing
  systems}}  \bibinfo{volume}{30} (\bibinfo{year}{2017}).
\newblock


\bibitem[Veitch et~al\mbox{.}(2019)]%
        {veitch2019using}
\bibfield{author}{\bibinfo{person}{Victor Veitch}, \bibinfo{person}{Yixin
  Wang}, {and} \bibinfo{person}{David Blei}.} \bibinfo{year}{2019}\natexlab{}.
\newblock \showarticletitle{Using embeddings to correct for unobserved
  confounding in networks}.
\newblock \bibinfo{journal}{\emph{Advances in Neural Information Processing
  Systems}}  \bibinfo{volume}{32} (\bibinfo{year}{2019}).
\newblock


\bibitem[Velickovic et~al\mbox{.}(2018)]%
        {VelickovicCCRLB18}
\bibfield{author}{\bibinfo{person}{Petar Velickovic}, \bibinfo{person}{Guillem
  Cucurull}, \bibinfo{person}{Arantxa Casanova}, \bibinfo{person}{Adriana
  Romero}, \bibinfo{person}{Pietro Li{\`{o}}}, {and} \bibinfo{person}{Yoshua
  Bengio}.} \bibinfo{year}{2018}\natexlab{}.
\newblock \showarticletitle{Graph Attention Networks}. In
  \bibinfo{booktitle}{\emph{6th International Conference on Learning
  Representations, {ICLR} 2018, Vancouver, BC, Canada, April 30 - May 3, 2018,
  Conference Track Proceedings}}.
\newblock


\bibitem[Wang et~al\mbox{.}(2020)]%
        {WangHK20}
\bibfield{author}{\bibinfo{person}{Hao Wang}, \bibinfo{person}{Hao He}, {and}
  \bibinfo{person}{Dina Katabi}.} \bibinfo{year}{2020}\natexlab{}.
\newblock \showarticletitle{Continuously Indexed Domain Adaptation}. In
  \bibinfo{booktitle}{\emph{Proceedings of the 37th International Conference on
  Machine Learning, {ICML} 2020, 13-18 July 2020, Virtual Event}}
  \emph{(\bibinfo{series}{Proceedings of Machine Learning Research},
  Vol.~\bibinfo{volume}{119})}. \bibinfo{publisher}{{PMLR}},
  \bibinfo{pages}{9898--9907}.
\newblock


\bibitem[Wang and Blei(2019)]%
        {wang2019blessings}
\bibfield{author}{\bibinfo{person}{Yixin Wang} {and} \bibinfo{person}{David~M
  Blei}.} \bibinfo{year}{2019}\natexlab{}.
\newblock \showarticletitle{The blessings of multiple causes}.
\newblock \bibinfo{journal}{\emph{J. Amer. Statist. Assoc.}}
  \bibinfo{volume}{114}, \bibinfo{number}{528} (\bibinfo{year}{2019}),
  \bibinfo{pages}{1574--1596}.
\newblock


\bibitem[Xu et~al\mbox{.}(2016)]%
        {xu2016bayesian}
\bibfield{author}{\bibinfo{person}{Yanbo Xu}, \bibinfo{person}{Yanxun Xu},
  {and} \bibinfo{person}{Suchi Saria}.} \bibinfo{year}{2016}\natexlab{}.
\newblock \showarticletitle{A Bayesian nonparametric approach for estimating
  individualized treatment-response curves}. In
  \bibinfo{booktitle}{\emph{Machine learning for healthcare conference}}. PMLR,
  \bibinfo{pages}{282--300}.
\newblock


\bibitem[Yao et~al\mbox{.}(2018)]%
        {yao2018representation}
\bibfield{author}{\bibinfo{person}{Liuyi Yao}, \bibinfo{person}{Sheng Li},
  \bibinfo{person}{Yaliang Li}, \bibinfo{person}{Mengdi Huai},
  \bibinfo{person}{Jing Gao}, {and} \bibinfo{person}{Aidong Zhang}.}
  \bibinfo{year}{2018}\natexlab{}.
\newblock \showarticletitle{Representation learning for treatment effect
  estimation from observational data}.
\newblock \bibinfo{journal}{\emph{Advances in Neural Information Processing
  Systems}}  \bibinfo{volume}{31} (\bibinfo{year}{2018}).
\newblock


\bibitem[Yoon et~al\mbox{.}(2018)]%
        {yoon2018ganite}
\bibfield{author}{\bibinfo{person}{Jinsung Yoon}, \bibinfo{person}{James
  Jordon}, {and} \bibinfo{person}{Mihaela Van Der~Schaar}.}
  \bibinfo{year}{2018}\natexlab{}.
\newblock \showarticletitle{GANITE: Estimation of individualized treatment
  effects using generative adversarial nets}. In
  \bibinfo{booktitle}{\emph{International Conference on Learning
  Representations}}.
\newblock


\bibitem[Yu et~al\mbox{.}(2022)]%
        {yu2022learning}
\bibfield{author}{\bibinfo{person}{Yue Yu}, \bibinfo{person}{Xuan Kan},
  \bibinfo{person}{Hejie Cui}, \bibinfo{person}{Ran Xu}, \bibinfo{person}{Yujia
  Zheng}, \bibinfo{person}{Xiangchen Song}, \bibinfo{person}{Yanqiao Zhu},
  \bibinfo{person}{Kun Zhang}, \bibinfo{person}{Razieh Nabi},
  \bibinfo{person}{Ying Guo}, {et~al\mbox{.}}} \bibinfo{year}{2022}\natexlab{}.
\newblock \showarticletitle{Learning Task-Aware Effective Brain Connectivity
  for fMRI Analysis with Graph Neural Networks}.
\newblock \bibinfo{journal}{\emph{arXiv preprint arXiv:2211.00261}}
  (\bibinfo{year}{2022}).
\newblock


\bibitem[Zang and Wang(2020)]%
        {ZangW20a}
\bibfield{author}{\bibinfo{person}{Chengxi Zang} {and} \bibinfo{person}{Fei
  Wang}.} \bibinfo{year}{2020}\natexlab{}.
\newblock \showarticletitle{Neural Dynamics on Complex Networks}. In
  \bibinfo{booktitle}{\emph{{KDD} '20: The 26th {ACM} {SIGKDD} Conference on
  Knowledge Discovery and Data Mining, Virtual Event, CA, USA, August 23-27,
  2020}}. \bibinfo{publisher}{{ACM}}, \bibinfo{pages}{892--902}.
\newblock


\end{thebibliography}

%%
%% If your work has an appendix, this is the place to put it.
\appendix
% \clearpage
\section{Appendix}

\subsection{Proofs of Theorems}

\subsubsection{Proof of Theorem.~\ref{theorem:t1} }\label{sec:proof_t1}

\begin{theorem_copy}
 Let $j\in \{0,1\}$ be the binary treatment values, and let $N$ and $T$ denote the number of units and observed timestamp lengths, respectively. Let $P_j^t = P(\mathbf{Z}^t\mid \mathbf{A}^t=j)$, be the distribution of latent representation $\mathbf{Z}^t$ for the group of units with treatments $j$ at time $t$. Let $f$, $\phi$, $d_\mathbf{A}^j$ be the initial state encoder, the ODE function of \odelower, and logits of predicting treatment $j$. The necessary and sufficient condition for the min-max game in Eq.~(\ref{eqn:a_pred}) to be optimal is $P_0^t = P_1^t, \forall t\in\left(t_0,t_1\cdots t_T\right)$.
 \end{theorem_copy}
  The proof of Theorem~\ref{theorem:t1} follows \cite{BicaAJS20, MelnychukFF22} and consists of two steps: to find the optimal $d_\mathbf{A}^{j}$ while fixing $f$ and $\phi$, and then to prove the optimal $f$ and $\phi$ while fixing $d_\mathbf{A}^{j}$ can balance the latent representations, i.e., $P_0^t = P_1^t$. The first step is given by Proposition~\ref{pro:1}.
\begin{proposition}\label{pro:1}
(Proposition 1 in~\cite{MelnychukFF22}) Let $\alpha_{j}=P(A^t=j)$. When the initial state encoder $f$ and ODE function $\phi$ are fixed, the optimal $d_\mathbf{A}^j$ at time $t$ is:
\begin{equation}
    {d_\mathbf{A}^{j}}^* = \frac{\alpha_{j} P_j^t}{\sum_{j'\in\{0,1\}} \alpha_{j'} P_{j'}^t}.
\end{equation}
\end{proposition}
\begin{proof}
 When fixing $f$ and $\phi$, ${d_\mathbf{A}^{j}}^*$ is obtained by:
 \begin{align}
    {d_\mathbf{A}^{j}}^* &= \underset{d_\mathbf{A}^j}{\text{argmin}}\sum_{j\in\{0,1\}}\mathds{1}_{(\mathbf{A}^t=j)}-\log\left(d_\mathbf{A}^j\left(r(\mathbf{Z}^t)\right)\right),\label{eqn:optim_da}\\
     &\qquad\qquad\text{subject to } \sum_{j\in\{0,1\}} d_\mathbf{A}^j\left(r(\mathbf{Z}^t)\right)=1\label{eqn:da_cons},
 \end{align}
where Eq.~(\ref{eqn:optim_da}) is adapted from Eq.~(\ref{eqn:a_pred}). Note Eq.~(\ref{eqn:optim_da}) can be applied to any $i$ and $t$ in Eq.~(\ref{eqn:a_pred}), so we disregard the expectation with respect to them. Let $\alpha_{j}=P(A^t=j)$ and $P_j^t = P(\mathbf{Z}^t\mid \mathbf{A}^t=j)$. Then Eq.~(\ref{eqn:optim_da}) can be rewritten as: 
\begin{align}
        {d_\mathbf{A}^{j}}^* &= \underset{d_\mathbf{A}^j}{\text{argmin}}\sum_{j\in\{0,1\}}-\mathbb{E}_{\mathbf{Z}^t\sim P_j^t}\alpha_{j}\log\left(d_\mathbf{A}^j\left(r(\mathbf{Z}^t)\right)\right)\label{eqn:optim_da2}\\
        & = \underset{d_\mathbf{A}^j}{\text{argmin}}\sum_{j\in\{0,1\}}-\int_{\mathbf{Z'}^t}\alpha_{j}\log\left(d_\mathbf{A}^j\left(r({\mathbf{Z'}^t})\right)\right)P_j^t d\mathbf{Z'}^t.\label{eqn:optim_da3}
\end{align}
We can also take pointwise optimization for any $\mathbf{Z'}^t$ in Eq.~(\ref{eqn:optim_da3})  Then by combining Equation (\ref{eqn:optim_da3}) with the constraint in Equation (\ref{eqn:da_cons}) and using Lagrange multipliers, we have: 
\begin{align}
    {d_\mathbf{A}^{j}}^* &= \underset{d_\mathbf{A}^j}{\text{argmin}}\sum_{j\in\{0,1\}}-\alpha_{j}\log\left(d_\mathbf{A}^j\left(r({\mathbf{Z'}^t})\right)\right)P_j^t \nonumber\\
    &+ \lambda\left(\sum_{j\in\{0,1\}} d_\mathbf{A}^j\left(r(\mathbf{Z}^t)\right)-1\right)\label{eqn:lagrange}.
\end{align}
Let $J=\sum_{j\in\{0,1\}}-\alpha_{j}\log(d_\mathbf{A}^j(r({\mathbf{Z'}^t})))P_j^t + \lambda(\sum_{j\in\{0,1\}} d_\mathbf{A}^j(r(\mathbf{Z}^t))-1)$ be the objective in Eq.~(\ref{eqn:lagrange}) The optimal values can be obtained by taking partial gradients $\frac{\partial J}{\partial d_\mathbf{A}^{j}} = 0$ and $\frac{\partial J}{\lambda} = 0$, respectively. By computing them jointly we can obtain ${d_\mathbf{A}^{j}}^* = \frac{\alpha_{j}P_j^t}{\sum_{j'\in\{0,1\}} \alpha_{j'}P_{j'}^t}$.
 \end{proof}

The second step is to prove Theorem.~\ref{theorem:t1} that the optimal $f$ and $\phi$ can obtain balanced representations with respect to treatments.
\begin{proof}
With Proposition~\ref{pro:1}, we can fix the optimal ${d_\mathbf{A}^{j}}^*$ and find the condition where Eq.~(\ref{eqn:a_pred}) achieves optimum. Putting ${d_\mathbf{A}^{j}}^*$ into the objective in Eq.~(\ref{eqn:a_pred}) and applying similar simplifications as in Eq.~(\ref{eqn:optim_da3}) and Eq.~(\ref{eqn:lagrange}), we have:
\begin{align}
    f^*,\phi^* &= \underset{f,\phi}{\text{argmax}}\sum_{j\in\{0,1\}}-\mathbb{E}_{\mathbf{Z}^t\sim P_j^t}\log\left(\frac{\alpha_{j}P_j^t}{\sum_{j'\in\{0,1\}} \alpha_{j'}P_{j'}^t}\right)\\
    &= \underset{f,\phi}{\text{argmin}}\sum_{j\in\{0,1\}}\mathbb{E}_{\mathbf{Z}^t\sim P_j^t}\log\left(\frac{P_j^t}{\sum_{j'\in\{0,1\}} \alpha_{j'}P_{j'}^t}\right) + \log(\alpha_{j})\\
    & = \underset{f,\phi}{\text{argmin}}\sum_{j\in\{0,1\}}\text{KL}\left(P_j^t \Bigg| \Bigg| {\sum_{j'\in\{0,1\}} \alpha_{j'}P_{j'}^t}\right)+ \log(\alpha_{j}),\\
\end{align}
where $\text{KL}(\cdot||\cdot)$ is the KL divergence. Note that $\sum_{j\in\{0,1\}}\log(\alpha_{j})$ is constant in observation data. $\text{KL}(\cdot||\cdot)\geq0$ and it reaches $0$ when the two operands are equal. Therefore, to have $f^*,\phi^*$, for $j\in\{0,1\}$, we have $P_0^t = P_1^t ={\sum_{j'\in\{0,1\}} \alpha_{j'}P_{j'}^t}$. Therefore, the optimal $f^*,\phi^*$ are those who achieve $P_0^t = P_1^t$. We can apply this to all the $N$ units and all the $T$ observed timestamps to obtain the global optimum of Eq.~(\ref{eqn:a_pred}), which concludes the proof of Theorem~\ref{theorem:t1}.    
\end{proof}

\subsubsection{Proof of Theorem~\ref{theorem:t2}}\label{sec:proof_t2}
\begin{theorem_copy}
Let $f$, $\phi$, $d_\mathbf{G}$ be the initial state encoder, the ODE function of \odelower, and the interference predictor. The necessary and sufficient condition for min-max game in Eq.~(\ref{eqn:pred_g}) to be optimal is $P(\mathbf{Z}^t,\mathbf{A}^t|\mathbf{G}^t=g')$ is identical for any $g'$.
\end{theorem_copy}

The proof of Theorem~\ref{theorem:t2} follows~\cite{WangHK20}. Similar to Theorem~\ref{theorem:t1}'s proof, it first finds the optimum of $d_\mathbf{G}$, and then proves that the optimal $f$ and $\phi$ can balance the representations with respect to interference. We first restate the Lemma 4.1 in~\cite{WangHK20} in Proposition.~\ref{pro:2}.

\begin{proposition}\label{pro:2}
(Lemma 4.1 in~\cite{WangHK20}) Let $\mathbf{C}^t = [\mathbf{Z}^t,\mathbf{A}^t]$ be the concatenation of $\mathbf{Z}^t$ and $\mathbf{A}^t$. When fixing initial state encoder $f$ and ODE function $\phi$, the optimal $d_\mathbf{G}$ at time $t$ is:
\begin{equation}
    d_\mathbf{G}^* = \mathbb{E}_{\mathbf{G}^t\sim p(\mathbf{G}^t|\mathbf{C}^t)}(\mathbf{G}^t)\label{eqn:find_dg}.
\end{equation}
\end{proposition}
\begin{proof}
    Eq.~(\ref{eqn:find_dg}) is adapted from Eq.~(\ref{eqn:pred_g}). We disregard the expectation with respect to $i$ and $t$ since Eq.~(\ref{eqn:find_dg}) is applicable to any $i$ and $t$. If fixing $f$ and $\phi$, the optimal $d_\mathbf{G}^*$ is given by:
\begin{align}
    d_\mathbf{G}^* &=\underset{d_\mathbf{G}}{\text{argmin}} \ \mathbb{E}_{(\mathbf{Z}^t,\mathbf{A}^t,\mathbf{G}^t)\sim p(\mathbf{Z}^t,\mathbf{A}^t,\mathbf{G}^t)}\left(d_\mathbf{G}\left(r([\mathbf{Z}^t,\mathbf{A}^t])\right)-\mathbf{G}^t\right)^2\\
    &= \underset{d_\mathbf{G}}{\text{argmin}} \ \mathbb{E}_{(\mathbf{C}^t,\mathbf{G}^t)\sim p(\mathbf{C}^t,\mathbf{G}^t)}\left(d_\mathbf{G}\left(r(\mathbf{C}^t)\right)-\mathbf{G}^t\right)^2\\
    &= \underset{d_\mathbf{G}}{\text{argmin}} \ \mathbb{E}_{\mathbf{C}^t\sim p(\mathbf{C}^t)} \mathbb{E}_{\mathbf{G}^t\sim p(\mathbf{G}^t\mid \mathbf{C}^t)}\left(d_\mathbf{G}\left(r(\mathbf{C}^t)\right)-\mathbf{G}^t\right)^2.\label{eqn:optim_dg3}
\end{align}
As the quadratic expansion in~\cite{WangHK20}, the optimal interference predictor is $d_\mathbf{G}^* = \mathbb{E}_{\mathbf{G}^t\sim p(\mathbf{G}^t\mid \mathbf{C}^t)}(\mathbf{G}^t)$. 
\end{proof}

Here we introduce and reformulate Theorem 4.1 in~\cite{WangHK20} as Lemma.~\ref{lemma:1}. 

\begin{lemma}\label{lemma:1}
(Theorem 4.1 in~\cite{WangHK20}) Given $\mathbb{E}_{x}\mathbb{V}(y\mid x)$ where $\mathbb{V}$ denotes variance, its global optimum can be achieved if and only if for any $x$, $\mathbb{E}(y\mid x) = \mathbb{E}(y)$.
\end{lemma}

With Proposition.~\ref{pro:2} and Lemma.~\ref{lemma:1}, we can prove Theorem.~\ref{theorem:t2}. 
\begin{proof}
Fixing the optimal $d_\mathbf{G}^*$ in Proposition.~\ref{pro:2}, the optimal $f$ and $\phi$ for objective Eq.~(\ref{eqn:optim_dg3}) is:
\begin{align}
    f^*,\phi^* &= \underset{f,\phi}{\text{argmax}}\ \mathbb{E}_{\mathbf{C}^t\sim p(\mathbf{C}^t)} \mathbb{E}_{\mathbf{G}^t\sim p(\mathbf{G}^t\mid \mathbf{C}^t)}\left(\mathbb{E}_{\mathbf{G}^t\sim p(\mathbf{G}^t|\mathbf{C}^t)}(\mathbf{G}^t)-\mathbf{G}^t\right)^2\\
    & = \underset{f,\phi}{\text{argmax}}\ \mathbb{E}_{\mathbf{C}^t\sim p(\mathbf{C}^t)}\mathbb{V}\left(\mathbf{G}^t\mid \mathbf{C}^t\right)\label{eqn:variance}.
\end{align}
Eq.~(\ref{eqn:variance}) has the same form as the equation in Lemma.~\ref{lemma:1}. Then substituting $x=\mathbf{C}^t$ and $y=\mathbf{G}^t$ in Lemma.~\ref{lemma:1}, we have $\mathbb{E}(\mathbf{G}^t\mid \mathbf{C}^t) = \mathbb{E}(\mathbf{G}^t)$. In other words, $\mathbf{G}^t \ind \mathbf{C}^t$. Therefore, we can also use the inverse form that $\mathbb{E}(\mathbf{C}^t) = \mathbb{E}(\mathbf{C}^t\mid \mathbf{G}^t) = \mathbb{E}([\mathbf{Z}^t,\mathbf{A}^t]\mid \mathbf{G}^t)$ for any $\mathbf{G}^t$, which concludes the proof of Theorem.~\ref{theorem:t2}.
\end{proof}

\subsection{Pseudo-Code of \model Training }\label{sec:pseudo}
The pseudo-code for training \model is shown in Algorithm.~\ref{algo:training}. Specifically, we use an alternative training trick to trade-off the outcome prediction and adversarial balancing. 

% \vspace{-0.5cm}
\begin{algorithm}[H]
\caption{The optimization process of \model}
\begin{algorithmic}
    \State \textbf{Input:} Multi-agaent dynamical system ${\mathcal{G}}$; the observational data $\left(\left( \mathbf{X}^t, \mathbf{A}^t, \mathbf{Y}^t \right) \cup \mathbf{V}\right)$; observed timestamps $\left(t_0,t_1\cdots t_T\right)$.
    \State \textbf{Output:} Trained initial state encoder $f(\cdot)$; \odelower function $\phi(\cdot)$; outcome decoder $d_\mathbf{Y}(\cdot)$; treatment predictor $d_\mathbf{A}(\cdot)$; interference predictor $d_\mathbf{G}(\cdot)$.
    \State \textbf{Training:}
    \State Initialize $f(\cdot)$, $\phi(\cdot)$, $d_\mathbf{Y}(\cdot)$, $d_\mathbf{A}(\cdot)$, $d_\mathbf{G}(\cdot)$;
\For {w = 1, 2, ..., W } \Comment{\com{Train W iterations}}
    \If {w\%(K+1) = 0}
    \Comment{\com{Train with $L^{\langle Y \rangle}$ for 1 steps}}
        \State Compute $L^{\langle Y \rangle}$;
        \State One step optimization for $f(\cdot)$, $\phi(\cdot)$, and $d_\mathbf{Y}(\cdot)$: 
        \State $\theta_{f}^{(w+1)} = \theta_{f}^{(w)}-\eta\nabla_{\theta_{f}} L^{\langle Y \rangle}$; \Comment{\com{$\eta$ is learning rate}}
        \State $\theta_{\phi}^{(w+1)} = \theta_{\phi}^{(w)}-\eta\nabla_{\theta_{\phi}} L{\langle Y \rangle}$; 
        \State $\theta_{d_\mathbf{Y}}^{(w+1)} = \theta_{d_\mathbf{Y}}^{(w)}-\eta\nabla_{\theta_{d_\mathbf{Y}}} L^{\langle Y \rangle}$; 
    \Else \Comment{\com{Train with $L$ for $K$ steps}}
        \State Compute $L$;
        \State One step optimization for $f(\cdot)$, $\phi(\cdot)$, $d_\mathbf{Y}(\cdot)$, $d_\mathbf{A}(\cdot)$, $d_\mathbf{G}(\cdot)$: 
        \State $\theta_{d_\mathbf{A}}^{(w+1)} = \theta_{d_\mathbf{A}}^{(w)}-\eta\nabla_{\theta_{d_\mathbf{A}}} L$; 
        \State $\theta_{d_\mathbf{G}}^{(w+1)} = \theta_{d_\mathbf{G}}^{(w)}-\eta\nabla_{\theta_{d_\mathbf{G}}} L$; 
        \State $\theta_{d_\mathbf{Y}}^{(w+1)} = \theta_{d_\mathbf{Y}}^{(w)}-\eta\nabla_{\theta_{d_\mathbf{Y}}} L$
        \State $\theta_{f}^{(w+1)} = \theta_{f}^{(w)}-\eta\nabla_{\theta_{f}} L^{\langle Y \rangle}+\eta\nabla_{\theta_{f}} L^{\langle A \rangle}+\eta\nabla_{\theta_{f}} L^{\langle G \rangle}$; 
        \State \Comment{\com{Gradient reversal}}
        \State $\theta_{\phi}^{(w+1)} = \theta_{\phi}^{(w)}-\eta\nabla_{\theta_{\phi}} L^{\langle Y \rangle}+\eta\nabla_{\theta_{\phi}} L^{\langle A \rangle}+\eta\nabla_{\theta_{\phi}} L^{\langle G \rangle}$; 
        \State \Comment{\com{Gradient reversal}}
\EndIf 
\EndFor
\State \textbf{Return} $f(\cdot)$, $\phi(\cdot)$, $d_\mathbf{Y}(\cdot)$, $d_\mathbf{A}(\cdot)$, $d_\mathbf{G}(\cdot)$.
\end{algorithmic}
\label{algo:training}
\end{algorithm}

\subsection{Experimental Settings}\label{sec:exp_setting}

\textbf{Datasets. } In observational data, we only have outcomes under one treatment trajectory but not the ground-truths of counterfactual outcomes. Therefore, we follow~\cite{ma2022learning,guo2020learning,veitch2019using} to use semi-synthetic data to evaluate \model. That is, the graph structure and node features are real, but treatments and potential outcomes are simulated. We use the social networks Flickr and BlogCatalog as in~\cite{guo2020learning,chu2021graph,ma2021deconfounding} as the graph $\mathcal{G}$. We follow these works to first encode the node features into low-dimensional embeddings ($10$-dimensional in this paper) via LDA~\cite{blei2003latent}. We then follow~\cite{jiang2022estimating} to use Metis~\cite{karypis1998fast} to split the graph into training/validation/testing sets. To simulate treatment and potential outcomes over time, ~\cite{geng2017prediction, BicaAJS20, SeedatIBQS22} use a longitudinal simulation environment, which, however, assumes the units are mutually independent. We extend it into our multi-agent dynamical systems setting by considering the neighbor confounders and interference. During the simulation, we are motivated by the vaccine's use case. Specifically, treatment $\mathbf{A}_i^t$ denotes getting a vaccine or not at time $t$ of unit $i$. The trajectory of $\mathbf{A}_i^t$ denotes the vaccine records over time, e.g., a unit may have a booster dose after the initial vaccine. In this case, the time-dependent covariates $\mathbf{X}^t$ could be the health condition, static covariates $\mathbf{V}$ could be race or educational background (assuming it does not change during the study) and potential outcome $\mathbf{Y}^t$ could be immunity to the virus. As discussed in Sec.~\ref{sec:problem}, the potential outcome $\mathbf{Y}^t$ is essentially a part of $\mathbf{X}^t$. This is a common setting in longitudinal causal inference studies~\cite{BicaAJS20,SeedatIBQS22,MelnychukFF22}. During the simulation, we follow this protocol: the health condition $\mathbf{X}^t_i$ has a value range $[0.1,10]$, in which a higher value means a better health condition. Meanwhile, a higher health condition means a lower probability to receive a vaccine (treatment).

\emph{Treatment simulation. } The treatment $\mathbf{A}_i^t$ is affected by a unit's own time-dependent covariates $\mathbf{X}_i^t$, static covariates $\mathbf{V}_i$ and those of their neighbors. Let $\mathbf{E}_i = w_a\mathbf{V}_i$ denote the effects of static confounders on treatments, where $w_a$ is a generated parameter representing this mechanism. The treatment is then simulated by Bernoulli generator with probability the $p_i^t (a)$ of unit $i$ at time $t$:

\begin{align}
	p_i^t (a) = \sigma \Biggl(&\underbrace{\gamma_a (\delta_a-\bar{\mathbf{X}}_i^t)}_{\text {time-dependent covariates}} + \underbrace{\gamma_n \Biggl(\delta_n-(\frac{1}{|N_i|}\sum_{j\in \mathcal{N}_i}\bar{\mathbf{X}}_j^t)\Biggr)}_{\text {Neighbor time-dependent covariates}} \nonumber \\
 &+\underbrace{\gamma_f\mathbf{E}i}_{\text {Static covariates}}  +\underbrace{\gamma_g(\frac{1}{|N_i|}\sum_{j\in \mathcal{N}_i}\mathbf{E}_j)}_{\text {Neighbor static covariates}}\Biggr),
\end{align}
where $\sigma(\cdot)$ is sigmoid function. $\gamma_a$, $\gamma_n$, $\gamma_f$ and $\gamma_g$ are degrees of time-dependent confounders, neighbor time-dependent confounders, static confounders and neighbor static confounders, respectively. The default values are $[\gamma_a, \gamma_n, \gamma_f, \gamma_g] = [10,3.3,10,3.3]$ ($\frac{\gamma_a}{\gamma_n}=\frac{\gamma_f}{\gamma_g}=3$), to mimic that a unit's own confounding factors should affect it more than neighbors. Note $\bar{\mathbf{X}}_i^t$ is the average time-dependent covariates until $t$. This reflects that past time-dependent covariates also affect the treatment. We set $\delta_a = \delta_n =5$ as adjustments. 

\emph{Potential outcome simulation. } We follow~\cite{geng2017prediction, BicaAJS20, SeedatIBQS22} to use a Pharmacokinetic-Pharmacodynamic (PK-PD) model~\cite{goutelle2008hill} to simulate the continuous trajectory. PK-PD model is a popular bio-mathematical model and a natural fit for our vaccine use case. As mentioned above, $\mathbf{Y}^t$ is essentially a part of $\mathbf{X}^t$. Therefore we directly simulate the trajectory of $\mathbf{X}_i^t$: 
\begin{align}
\frac{d\mathbf{X}^t_i}{dt}= \mathbf{X}^t_i\bigg(&\underbrace{\rho_u \log \left(\frac{K}{\mathbf{X}^t_i}\right)}_{\text {Time-dependt covariates}}+\underbrace{\rho_n \log \left(\frac{K}{\mathbf{X}^t_i}\right)}_{\text {Neighbor time-dependt covariates}}\nonumber \\
& + \underbrace{\rho_f\mathbf{O}_i}_{\text {Static covariates}} + \underbrace{\rho_g\sum_{j\in \mathcal{N}_i}\mathbf{O}_j)}_{\text {Neighbor static covariates}}\nonumber \\
&+\underbrace{\beta_{a} \mathbf{D}_i^t}_{\text {Treatment }}+\underbrace{\frac{1}{|N_i|}\sum_{j\in \mathcal{N}_i}\beta_{n} \mathbf{D}_j^t}_{\text {Interference}}+\underbrace{e_{i}^t}_{\text {Noise }} \bigg),
\label{eqn:pkpd}
\end{align}
where $K$ controls the effects of time-dependent covariates on future potential outcomes. $\mathbf{O}_i = w_x\mathbf{V}_i$ denote the effects of static confounders on potential outcomes, where $w_x$ represents this mechanism. $\rho_u$, $\rho_n$, $\rho_f$ and $\rho_g$ are degrees of time-dependent covariates, neighbor time-dependent covariates, static covariates and neighbor static covariates, respectively. Their default values are $[\rho_u, \rho_n, \rho_f, \rho_g] = [-0.001, -00033,0.001,0.00033]$ ($\frac{\rho_u}{\rho_n} = \frac{\rho_f}{\rho_g}=3$). $\beta_a$ and $\beta_n$ control the strengths of treatment and interference. We set them as $[\beta_a,\beta_n] = [0.03, 0.01]$. The values also reflect that a unit's own covariates and treatment should have stronger effects on its future potential outcomes than neighbors. In reality, the effects of vaccines on providing protection decrease over time. To mimic this phenomenon, we use the following decay function to model the effects of treatments over time as in~\cite{BicaAJS20,SeedatIBQS22}.
\begin{align}
    \mathbf{D}_i^t = \tilde{\mathbf{D}}_i^t + \mathbf{D}_i^{(t-1)}/2,
\end{align}
where $\mathbf{D}_i^t$ denote the protection effect at time $t$, and $\tilde{\mathbf{D}}_i^t$ means a full protection of vaccines given at $t$. In other words, the unit receives a vaccine at time $t$, i.e., $\mathbf{A}_i^t=1$. We set $\tilde{\mathbf{D}}_i^t=1$.

\subsection{Notation Table}\label{sec:notations}
The notations and their corresponding description are in Table.~\ref{tb:notation_table}.

\begin{table}[H]
\caption{Notations.}
% \fontsize{8}{10
% }\selectfont
\vspace{-10pt}
\setlength{\tabcolsep}{2.1pt}
\begin{tabular}{@{}ll@{}}
\toprule[1.1pt]
Notation & Description \\ \midrule
$\mathcal{G}$ & Graph that represents multi-agent dynamical system\\
$\mathcal{V}$ & Node set in $\mathcal{G}$\\
$\mathcal{E}$ & Edge set in $\mathcal{G}$\\
$\mathbf{V}$ & Static unit features\\
$\mathbf{X}^t$ & Time-dependent covariates at time $t$\\
$\mathbf{A}^t$ & Treatment at time $t$\\
$\mathbf{Y}^t$ & Observed outcomes at time $t$\\
$\mathbf{Y}^t(A^t=a)$ & Potential outcomes under treatment $a$\\
$\mathbf{\Bar{X}}^t$ & Time-dependent covariates collections up to $t$\\
$\mathbf{\Bar{A}}^t$ & Treatment collections up to $t$\\
$\mathbf{\Bar{Y}}^t$ & Observed outcomes collections up to $t$\\
$\mathcal{H}^t$ & Past observations\\
$\mathbf{A}^t_{\mathcal{N}_i}$ & Treatments of node $i$'s first-order neighbors\\
$\mathbf{A}^t_{\mathcal{N}_{-i}}$ & Treatments of nodes that are \\&beyond $i$'s first-order neighbors\\
$\mathbf{G}_i^t$ & Interference summary variable\\
$\mathbf{Z}_i^t$ & Continuous latent trajectory for node $i$\\
$\mathbf{C}_i^t$ & Concatenation of $\mathbf{Z}_i^t$ and $\mathbf{A}_i^t$\\
$g(\cdot)$ & Interference summary function\\
$\phi(\cdot)$ & ODE function\\
$f(\cdot)$ & Initial state encoder function\\
$r(\cdot)$ & Gradient reversal layer\\
$d_\mathbf{Y}(\cdot)$ & Outcome prediction layer\\
$d_\mathbf{A}(\cdot)$ & Treatment prediction layer\\
$d_\mathbf{G}(\cdot)$ & Interference prediction layer\\
$N$ & Number of nodes in $\mathcal{V}$\\
$T$ & Number of observed timestamps\\
$L^{\langle Y \rangle}$ & Loss function of outcome prediction\\
$L^{\langle A \rangle}$ & Loss function of treatment prediction\\
$L^{\langle G \rangle}$ & Loss function of interference prediction\\
$\alpha_{\mathbf{A}}$& Weight of treatment balancing\\
$\alpha_{\mathbf{G}}$& Weight of interference balancing\\
\bottomrule[1.1pt]
\end{tabular}
\label{tb:notation_table}
% \vspace{-15pt}
\end{table}
\end{document}